%% file: main.tex
\def\final{1}
\def\submission{0}

\ifnum\final=0 
\documentclass[twoside]{article}
\usepackage[]{aistats2019}
	\usepackage{times}
\usepackage[margin=.88in]{geometry}
\else
\documentclass[11pt]{article}
 \usepackage{kpfonts}
                \DeclareMathAlphabet{\mathsf}{OT1}{cmss}{m}{n}
                \SetMathAlphabet{\mathsf}{bold}{OT1}{cmss}{bx}{n}
                
                \usepackage[margin=1.01in]{geometry}
\fi

\usepackage{algorithm, algorithmic}
\usepackage{hyphenat}
\usepackage{amsmath, amssymb, amsthm, bm}
\usepackage[inline]{enumitem}
\usepackage{framed}
\usepackage{verbatim}
\usepackage{color}
\usepackage{microtype}
\usepackage{multirow, tabularx}
\definecolor{DarkGreen}{rgb}{0.2,0.6,0.2}
\definecolor{DarkRed}{rgb}{0.6,0.2,0.2}
\definecolor{DarkBlue}{rgb}{0.15,0.15,0.55}
\definecolor{DarkPurple}{rgb}{0.4,0.2,0.4}

\usepackage[pdftex]{hyperref}
\hypersetup{
    linktocpage=true,
    colorlinks=true,				% false: boxed links; true: colored links
    linkcolor=DarkBlue,		% color of internal links
    citecolor=DarkBlue,	% color of links to bibliography
    urlcolor=DarkBlue,		% color of external links
}

\newcolumntype{Y}{>{\centering\arraybackslash}X}

\newcommand{\pr}[2]{\underset{#1}{\mathbb{P}}\left[ #2 \right]}
\newcommand{\ex}[2]{\underset{#1}{\mathbb{E}}\left[ #2 \right]}

\newcommand{\norm}[1]{\left\| #1 \right\|}

\newcommand{\eps}{\epsilon}

\newcommand{\R}{\mathbb{R}}
\newcommand{\cJ}{\mathcal{J}}
\newcommand{\bA}{\mathbf{A}}

\newcommand{\B}{\mathbb{B}}

\newcommand{\cI}{\mathcal{I}}

\newcommand{\bp}{\mathbf{p}}
\newcommand{\hbp}{\widehat{\mathbf{p}}}
\newcommand{\hn}{\hat{n}}
\newcommand{\cC}{\mathcal{C}}

\newcommand{\cN}{\mathcal{N}}
\newcommand{\cO}{\mathcal{O}}

\newcommand{\cQ}{\mathcal{Q}}
\newcommand{\cR}{\mathcal{R}}
\newcommand{\cRS}{\mathcal{R}^{\mathsf{RejSamp}}}
\newcommand{\cRG}{\mathcal{R}^{\mathsf{Gauss}}}
\newcommand{\cS}{\mathcal{S}}

\newtheorem{theorem}{Theorem}[section]
\newtheorem{thm}[theorem]{Theorem}

\newtheorem{lem}[theorem]{Lemma}
\newtheorem{fact}[theorem]{Fact}

\newtheorem{corollary}[theorem]{Corollary}

\theoremstyle{definition}
\newtheorem{defn}[theorem]{Definition}
\newtheorem{definition}[theorem]{Definition}

\newcommand{\prot}{\mathsf{Prot}}
\newcommand{\gauss}{\mathsf{Gauss}}

\newcommand{\samp}{\mathsf{RejSamp}}
\newcommand{\had}{\mathsf{HR}}
\newcommand{\phad}{\mathsf{PHR}}
\newcommand{\simp}{\mathsf{Simplex}}
\newcommand{\adapsamp}{\mathsf{AdSamp}}
\newcommand{\bx}{\mathbf{x}}

\newcommand{\by}{\mathbf{y}}
\newcommand{\iden}{\mathbb{I}}
\newcommand{\hby}{\hat{\mathbf{y}}}
\newcommand{\tby}{\tilde{\mathbf{y}}}
\newcommand{\ty}{\tilde{y}}

\newcommand{\ba}{\mathbf{a}}
\newcommand{\bz}{\mathbf{z}}

\newcommand{\be}{\mathbf{e}}
\newcommand{\bq}{\mathbf{q}}
\newcommand{\bw}{\mathbf{w}}

\newcommand{\bzero}{\mathbf{0}}
\newcommand{\bone}{\mathbf{1}}
\newcommand{\Ber}{\mathsf{Ber}}

\newcommand{\err}{\mathrm{err}}
\newcommand{\emperr}{\widehat{\err}}
\newcommand{\gd}{\mathsf{Good}}

\ifnum\final=1

\title{Linear Queries Estimation with Local Differential Privacy}
\author{Raef Bassily\\Department of Computer Science and Engineering\\
The Ohio State University\\ \textit{bassily.1@osu.edu}}

\date{}
\fi
\begin{document}
\ifnum\final=1
\maketitle
\fi

\ifnum\submission=1
\twocolumn[

\aistatstitle{Linear Queries Estimation with Local Differential Privacy}
	
	%Local Private Protocols for Linear Queries and Distribution Estimation in Low and High Dimensions

\aistatsauthor{ Anonymous Author(s)}

\aistatsaddress{  } ]
\fi

\ifnum\final=1
\begin{abstract}
    We study the problem of estimating a set of $d$ linear queries over some unknown distribution based on a sensitive data set under the constraint of local differential privacy (LDP). Let $\cJ$ be a data domain of size $J$. A linear query is uniquely identified by a vector $\bq\in\R^J,$ and is defined as the linear function $\langle \bq, ~\cdot \rangle: \simp(J)\rightarrow \R$, where $\simp(J)$ is the probability simplex in $\R^J$. Given a set $D=\{v_i \in \cJ: ~i\in [n]\}$ of private data items of $n$ individuals drawn i.i.d. from some unknown distribution $\bp\in\simp(J)$, we wish to estimate the values of a set of $d$ linear queries $\bq_1, \ldots, \bq_d$ over $\bp$ under LDP. This problem subsumes a wide range of estimation tasks including distribution estimation and $d$-dimensional mean estimation. We provide new algorithms for both the offline (non-adaptive) and the adaptive versions of this problem. 
    
    In the offline setting, the set of queries are determined and fixed at the beginning of the algorithm. 
    In the regime where $n\lesssim d^2/\log(J)$, our algorithms have $L_2$ estimation error (with respect to the distribution $\bp$) that is independent of $d$, and is tight up to a factor of $\tilde{O}\left(\log^{1/4}(J)\right)$.  Our algorithms combine different ideas such as $L_2$ projection on convex polytopes and rejection sampling. For the special case of distribution estimation, we show that projecting the output estimate of an algorithm due to \cite{acharya2018communication} on the probability simplex yields an $L_2$ error that depends only sub-logarithmically on $J$ in the  regime where $n\lesssim J^2/\log(J)$. These results show the possibility of accurate estimation of linear queries in the high-dimensional settings under the $L_2$ error criterion.

    In the adaptive setting, the queries are generated over $d$ rounds; one query at a time. At the start of each round $k\in [d],$ a query $\bq_k$ can be chosen \emph{adaptively} based on all the history of previous queries and answers. We give an algorithm for this problem with optimal $L_{\infty}$ estimation error (worst error in the estimated values for the queries w.r.t. the data distribution). Our bound matches a lower bound on the $L_{\infty}$ error for the \emph{offline} version of this problem \cite{DJW13}. %Our construction requires only constant communication per user. 

    %Our bound in this case is tight up to a sub-logarithmic factor in $J$.
    
    %by adding an extra projection step. We show that in the high-dimensional regime where $n\ll J$ ($d=J$ in this case),  this extra step , which shows the possibility of accurate distribution estimation in this setting.  
    
    %In particular, in the high-dimensional regime, where $n\ll d$, the resulting $L_2$ error is independent of $d$.
    
\end{abstract}

\else

\begin{abstract}
	
	We study the problem of estimating a set of $d$ linear queries with respect to some unknown distribution $\bp$ over a domain $\cJ=[J]$ based on a sensitive data set of $n$ individuals under the constraint of \emph{local differential privacy}. This problem subsumes a wide range of estimation tasks, e.g., distribution estimation and $d$-dimensional mean estimation. We provide new algorithms for both the offline (non-adaptive) and adaptive versions of this problem. 
	
	In the offline setting, the set of queries are  fixed before the algorithm starts. 
	In the regime where $n\lesssim d^2/\log(J)$, our algorithms attain $L_2$ estimation error that is independent of $d$. For the special case of distribution estimation, we show that projecting the output estimate of an algorithm due to \cite{acharya2018communication} on the probability simplex yields an $L_2$ error that depends only sub-logarithmically on $J$ in the regime where $n\lesssim J^2/\log(J)$. Our bounds are within a factor of at most $\left(\log(J)\right)^{1/4}$ from the optimal $L_2$ error when $n\lesssim d^2/\log(J)$. These results show the possibility of accurate estimation of linear queries in the high-dimensional settings under the $L_2$ error criterion.

	In the adaptive setting, the queries are generated over $d$ rounds; one query at a time. In each round, a query can be chosen \emph{adaptively} based on all the history of previous queries and answers. We give an algorithm for this problem with optimal $L_{\infty}$ estimation error (worst error in the estimated values for the queries w.r.t. the data distribution). Our bound matches a lower bound on the $L_{\infty}$ error for the \emph{offline} version of this problem \cite{DJW13}.

	%In the high dimensional regime ($n\ll d$)
	%Our construction requires only constant communication per user. 

%Our algorithms for this setting attain nearly optimal $L_2$ estimation error (w.r.t. the distribution $\bp$) in all ranges of the parameters $d, n$. In particular, in the high-dimensional regime, where $n\ll d$, the resulting $L_2$ error is independent of $d$. 

%Our bound in this case is tight up to a sub-logarithmic factor in $J$.

%For the special case of distribution estimation, our construction is based on a simple tweak of the construction of \cite{acharya2018communication}. We show that in the high-dimensional regime where $n\ll J$ ($d=J$ in this case), this tweak yields a $L_2$ error that depends only sub-logarithmically on $J$, which shows the possibility of accurate distribution estimation in this setting.  
%Our algorithms combines different ideas such as $L_2$ projection on convex polytopes and rejection sampling.
	%Moreover, our construction requires only constant communication per user. 
\end{abstract}

\fi

\input{intro}

\input{prelim}

\input{Offline}
\input{distrib}
\input{adaptive}

%\ifnum\short=1
%\vfill\eject \small

\ifnum\submission=1
\newpage
\fi

\bibliographystyle{alpha}
\bibliography{references}
%\else
%
%\addcontentsline{toc}{section}{References}
%\bibliographystyle{alpha}
%\bibliography{references}
%
%\fi

\end{document}

%% file: intro.tex
\section{Introduction}

%{\color{red} add usual intro, motiv for LDP}

Differential privacy \cite{DMNS} is a rigorous mathematical definition that has emerged as one of the most successful notions of privacy in statistical data analysis. Differential privacy provides a rich and powerful algorithmic framework for private data analysis, which can help organizations mitigate users' privacy concerns. There are two main models for private data analysis that are studied in the literature of differential privacy: the centralized model and the local model. The centralized model assumes a trusted centralized curator that collects all the personal information and then analyzes it. In contrast, the \emph{local model}, which dates back to \cite{W65}, does not involve a central repository. Instead, each individual holding a piece of private data randomizes her data herself via a local randomizer before it is collected for analysis. This local randomizer is designed to satisfy differential privacy, providing a strong privacy protection for each individual. The local model is attractive in many practical and industrial domains since it relieves organizations and companies from the liability of holding and securing their users private data. Indeed, in the last few years there have been many successful deployments of local differentially private algorithms in the industrial domain, most notably by Google and Apple \cite{erlingsson2014rappor, thakurta2017learning}.

In this paper, we study the problem of linear queries estimation under local differential privacy (LDP). Let  $\mathcal{J}=[J]$ be a data domain of size $J$. A linear query with respect to $\mathcal{J}$ is uniquely identified by a vector $\bq\in \R^J$ that describes a linear function $\langle \bq,~ \cdot \rangle: \simp(J)\rightarrow \R,$ where $\simp(J)$ denotes the probability simplex in $\R^J$. In this problem, we have a set of $n$ individuals (users), where each user $i\in [n]$ holds a private value $v_i\in \cJ$  drawn independently from some \emph{unknown} distribution $\bp\in\simp(J)$. An entity (server) generates a sequence of linear queries $\bq_1, \ldots, \bq_d$ and wishes to estimate, within a small error, the values of these queries over the unknown distribution $\bp$, i.e., $\langle \bq_1, ~\bp\rangle, \ldots, \langle \bq_d, ~\bp\rangle$. To do this, the server collects signals from the users about their inputs and use them to generate these estimates. Due to privacy concerns, the signal sent by each user is generated via a local randomizer that outputs a randomized (privatized) version of the user's true input in a way that satisfies LDP. The goal is to design a protocol that enables the server to derive accurate estimates for its queries under the LDP constraint. This problem subsumes a wide class of estimation tasks under LDP, including distribution estimation studied in \cite{DJW13, BS15, NIPS2015_5713, kairouz2016discrete, BNST17,ye2018optimal, acharya2018communication} and mean estimation in $d$ dimensions \cite{duchi2013local, DJW13}.

\paragraph{Non-adaptive versus Adaptive Queries:} In this work, we consider two versions for the above problem. In the non-adaptive (\emph{offline}) version, the set of $d$ queries $\bq_1, \ldots, \bq_d$ are decided by the server before the protocol starts (i.e., before users send their signals). In this case, the set of $d$ queries can be represented as the rows of a matrix $\bA\in \R^{d\times J}$ that is published before the protocol starts. In the \emph{adaptive} version of this problem, the $d$ queries are submitted and answered over $d$ rounds: one query in each round. Before the start of each round $k\in [d],$ the server can \emph{adaptively} choose the query $\bq_k$ based on all the history it sees, i.e., based on all the previous queries and signals from users in the past $k-1$ rounds. This setting is clearly harder than the offline setting. Both distribution estimation and mean estimation over a finite (arbitrary large) domain can be viewed as special cases of the offline queries model above. In particular, for distribution estimation, the queries matrix $\bA$ is set to $\iden_J$, the identity matrix of size $J$ (in such case, the dimensionality $d=J$). For $d$-dimensional mean estimation, the columns of $\bA$ are viewed as the set of all realizations of a $d$-dimensional random variable.  

One of the main challenges in the local model is dealing with high-dimensional settings (i.e., when $d\gtrsim n$). Previous constructions for distribution estimation \cite{DJW13, kairouz2016discrete, ye2018optimal, acharya2018communication} and mean estimation \cite{DJW13} suffer from an explicit polynomial dependence on the dimensions in the resulting $L_2$ estimation error.  

In this work, we address this challenge and give new constructions for large, natural families of offline linear queries that subsumes the above estimation problems. The resulting $L_2$ estimation error\footnote{In this work, we consider the true population risk not the empirical risk. We refer to it as the estimation error and sometimes as the \emph{true} error.} has no dependence on $d$ in the high-dimensional setting and depends only sub-logarithmically on $J$. We also consider the adaptive version of the general linear queries problem, and give a new protocol with optimal $L_{\infty}$ error (which is a more natural error criterion in the adaptive setting). We discuss these results below.

% In the offline setting, under a natural condition the queries matrix $\bA$, on  yield almost dimension-independent $L_2$ error in the high-dimensional setting. For the special case of distribution estimation under $L_2$ error criterion, we tweak the construction of \cite{acharya2018communication} to obtain a protocol with only sub-logarithmic (rather than polynomial) dependence on $J$ in the high-dimensional case (in particular, when $n\lesssim J^2$, which subsumes the case where $n\lesssim J$). Hence, our constructions for the offline case enable accurate estimation of linear queries in the high dimensional case, which was not feasible via adoption of previous constructions. Our constructions attain nearly optimal $L_2$ estimation error. 
%
%We also consider the adaptive version of the general linear queries problem, and give a new protocol with optimal $L_{\infty}$ error (which is a more natural error criterion in the adaptive setting). In particular, our derived upper bound on the $L_{\infty}$ error matches a known lower bound for the \emph{offline} version of this problem \cite{DJW13}. Moreover, our protocol is efficient and entails only constant amount of communication per user.

\subsection{Results and comparison to previous works}

The accuracy guarantees of our $\eps$-LDP protocols are summarized in Table~\ref{table}.

\paragraph{General offline linear queries:} We assume that the $L_2$ norm of any column of the queries matrix $\bA\in \R^{d\times J}$ is bounded from above by some arbitrary constant $r>0.$ We note that this is weaker assumption than assuming that the spectral norm of $\bA$ (largest singular value) is bounded by $r$. For any $r>0$, let $\cC_2(r)$ denote the collection of all matrices in $\R^{d\times J}$ satisfying this condition. We design $\eps$-LDP protocol that given any queries matrix $\bA$ from this family, it outputs an estimate for $\bA\bp$ with nearly optimal $L_2$ estimation error (see Section~\ref{sec:acc-def-offline} for the definition of the $L_2$ estimation error). As noted earlier, the resulting $L_2$ estimation error does not depend on $d$ in the high-dimensional setting: in particular, in the case where  $n\lesssim d^2/\log(J)$ (which subsumes the high-dimensional setting when $\log(J)\lesssim d$). This improves over the upper bound in \cite[Proposition~3]{DJW13} achieved by the ball sampling mechanism proposed therein. The near optimality of our protocol follows from the lower bound in the same reference (see Table~\ref{table}).
\noindent To construct our protocol, we start with an $(\eps, \delta)$-LDP protocol that employs the Gaussian mechanism together with the projection technique similar to the one used in \cite{NTZ12} in the \emph{centralized} model of differential privacy. We show the applicability of this technique in the local model. Next, we transform our $(\eps, \delta)$-LDP construction into a pure $\eps$-LDP construction while maintaining the same accuracy (and the same computational cost). To do this, we give a technique based on rejection sampling ideas from \cite{BS15, BNS18}. In particular, our technique can be viewed as a simpler, more direct version of the generic transformation of \cite{BNS18} tuned to the linear queries problem. For this general setting, we focus on improving the estimation error. We do not consider the problem of optimizing communication or computational efficiency. We think that providing a succinct description of the queries matrix (possibly under more assumptions on its structure) is an interesting problem, which we leave to future work.

\begin{table*}[ht!]
	
	\begin{center} \small
		%\begin{threeparttable}
		\begin{tabular}{|c||c|c|c|}
			\hline
			{\bf Problem/Error metric} & \begin{tabular}{@{}c@{}} Upper bound \\({\color{blue}This work}) \end{tabular} & \begin{tabular}{@{}c@{}} Upper bound \\ (Previous work) \end{tabular}& Lower bound \\
			\hline
			\hline  
			\begin{tabular}{@{}c@{}} General offline queries  \\ ($L_2$ error)\end{tabular} &  $r\cdot \min\left(\left(\frac{\log(J)\log(n)}{n\epsilon^2}\right)^{1/4},~ \sqrt{\frac{d}{n\epsilon^2}}\right)$ & \begin{tabular}{@{}c@{}} \\ $r\cdot\sqrt{\frac{d}{n\epsilon^2}}$ \\ \\ \cite[Prop.~3]{DJW13}\\  
			\end{tabular}& \begin{tabular}{@{}c@{}} \\ $r\cdot \min\left(\left(\frac{1}{n\epsilon^2}\right)^{1/4},~ \sqrt{\frac{d}{n\epsilon^2}}\right)$\\ \\ (\cite[Prop.~3]{DJW13}) \\ \end{tabular}\\
			\hline
			\begin{tabular}{@{}c@{}} Distribution estimation  \\ ($L_2$ error)\end{tabular} &  $\min\left(\left(\frac{\log(J)}{n\epsilon^2}\right)^{1/4},~ \sqrt{\frac{J}{n\epsilon^2}}\right)$ & \begin{tabular}{@{}c@{}} \\ $\sqrt{\frac{J}{n\epsilon^2}}$ \\ \\ \cite[Thm.~3]{acharya2018communication}\\  
			\end{tabular}& \begin{tabular}{@{}c@{}} \\ $\min\left(\left(\frac{1}{n\epsilon^2}\right)^{1/4},~\sqrt{\frac{J}{n\epsilon^2}}\right)$ \\ \\ (\cite{DJW13, ye2018optimal}) \\ \end{tabular}\\
			\hline
			\begin{tabular}{@{}c@{}} General adaptive queries \\ ($L_{\infty}$ error)\end{tabular} &  $r\,~\sqrt{\frac{c^2_{\eps} d\log(d)}{n}}$ & -- & \begin{tabular}{@{}c@{}} \\ $r\,~\sqrt{\frac{c^2_{\eps} d\log(d)}{n}}$\\ \\ (\cite[Prop.~4]{DJW13} \\ for offline queries) \\ \end{tabular}\\
			\hline
		\end{tabular} 
		%\end{threeparttable}
	\end{center}
	\caption{Error bounds for the proposed $\eps$-LDP protocols with comparison to previous results. 
		%We assume, w.l.o.g., that $n>\log(n) \log(d)$, that is, the error of our protocol is less than the trivial error. 
		Since the error in each case cannot exceed the trivial error $r$, each upper bound should be understood as the $\min$ of the stated bound and $r$.}
	\label{table}
\end{table*}

\paragraph{Distribution estimation:} For this special case, we extend the Hadamard-Response protocol of \cite{acharya2018communication} to the high-dimensional setting. This protocol enjoys several computational advantages, particularly, $O(\log(J))$ communication and running time for each user. We show that this protocol when combined with a projection step onto the probability simplex
%give a pure $\eps$-LDP construction that can be viewed as a simple tweak of the Hadamard-response protocol proposed in \cite{acharya2018communication}. In particular, we apply the projection technique on the top of their protocol. The main contribution in this part is that we show that this extra step 
gives $L_2$ estimation error that depends only sub-logarithmically on $J$ for all $n\lesssim J^2/\log(J)$. The resulting error is also tight up to a sub-logarithmic factor in $J$. We note that the $L_2$ error bound in \cite{acharya2018communication} is applicable only in the case where $n\gtrsim J/\eps^2$. Our result thus shows the possibility of accurate distribution estimation under the $L_2$ error criterion in the high-dimensional setting. Our bound also improves over the bound of \cite{acharya2018communication} for all $n\lesssim \frac{J^{2}}{\eps^2 \log(J)}$. To the best of our knowledge, existing results do not imply $L_2$ error bound better than the trivial $O(1)$ error in the regime where $n\lesssim \frac{J}{\eps^2}$. 
\noindent It is worthy to point out that the $L_2$ error bound of \cite{acharya2018communication} is optimal only when $n\gtrsim J^2/\eps^2$. Although this condition is not explicitly mentioned in \cite{acharya2018communication}, however, as stated in the same paper, their claim of optimality follows from the lower bound in \cite{ye2018optimal}; specifically, \cite[Theorem IV]{ye2018optimal}. From this theorem, it is clear that the lower bound is only valid when $n\geq \text{const.}\,\frac{J^2}{\eps^2}$. Hence, our bound does not contradict with the results of these previous works. We also note that the idea of projecting the estimated distribution onto the probability simplex was proposed in \cite{kairouz2016discrete} (along with a different protocol than that of \cite{acharya2018communication}). Although \cite{kairouz2016discrete} show empirically that the projection technique yield improvements in accuracy, no formal analysis or guarantees were provided for the resulting error in this case.

\noindent\textit{Note} that the  $L_2$ estimation error bounds in the previous works were derived for the expected $L_2$-\emph{squared} error, and hence the expressions here are the square-root of the bounds appearing in these references. Moreover, we note that our bounds are obtained by first deriving bounds on the  $L_2$-\emph{squared} estimation error, which then imply our stated bounds on the $L_2$ error. Hence, squaring our bounds give valid bounds on the $L_2$-squared error. 

\paragraph{Adaptive linear queries:} We assume the following constraint on any sequence of adaptively chosen queries $\langle \bq_1,~ \cdot \rangle, \ldots, \langle \bq_d, ~\cdot \rangle$: for each $k\in [d],$ $\|\bq_k\|_{\infty}\leq r$ for some  $r>0$. That is, each vector $\bq$ defining a query has a bounded $L_{\infty}$ norm. Unlike the offline setting, since the sequence of the queries is not fixed beforehand (i.e., the queries matrix $\bA$ is not known a priori), the above $L_{\infty}$ constraint is more natural than constraining a quantity related to the norm of the queries matrix as we did in the offline setting. For any $r>0$, we let $\cQ_{\infty}(r)=\left\{\langle \bq, \cdot \rangle:~ ~\|\bq\|_{\infty}\leq r\right\}$, i.e., $\cQ_{\infty}(r)$ denote the family of all linear queries satisfying the above constraint. In this setting, we measure accuracy in terms of the true $L_{\infty}$ error; that is, the maximum true error $\max\limits_{k\in[d]} \lvert y_k - \langle \bq_k, ~\bp\rangle \rvert$ in any of the estimates $\{y_k: ~k\in[d]\}$ for the $d$ queries. (See Section~\ref{sec:acc-def-adaptive} for a  precise definition). 

%described by vectors in $\R^J$ whose $L_{\infty}$ norm is bounded by $r$

\noindent We give a construction of $\eps$-LDP protocol that answers any sequence of $d$ adaptively chosen queries from $\cQ_{\infty}(r)$. Our protocol attains the optimal $L_{\infty}$ estimation error. The optimality follows from the fact that our upper bound matches a lower bound on the same error in the \emph{non-adaptive} setting given in \cite[Proposition~4]{DJW13}. In our protocol, each user sends only a constant number of bits to the server, namely, $O(\log(r))$ bits$/$user. In our protocol, the set of users are partitioned into $d$ disjoint subsets, and each subset is used to answer one query. Roughly speaking, this partitioning technique can be viewed as some version of sample splitting. %In the centralized model of differential privacy, sample-splitting is generally sub-optimal. 
In contrast, this technique is known to be suboptimal (w.r.t. the $L_{\infty}$ estimation error) in the \emph{centralized} model of differential privacy \cite{BNS+16}. %Our result shows that for adaptive linear queries in the local model, this technique is optimal. 
Moreover, given the offline lower bound in \cite{DJW13}, our result shows that adaptivity does not pose any extra penalty in the \emph{true} $L_{\infty}$ estimation error for linear queries in the local model. In contrast, it is still not clear whether the same statement can be made in the \emph{centralized} model of differential privacy. For instance, assuming $\eps = \Theta(1)$ and $n\gtrsim d^{3/2},$ then in the \emph{centralized} model, the best known upper bound on the \emph{true} $L_{\infty}$ estimation error for this problem in the \emph{adaptive} setting is $\approx d^{1/4}/\sqrt{n}$~ \cite[Corollary~6.1]{BNS+16} (which combines \cite{DMNS} with the generalization guarantees of differential privacy). Whereas in the offline setting, the \emph{true} $L_{\infty}$ error is upper-bounded by $\approx \sqrt{\frac{\log(d)}{n}}$ (combining \cite{DMNS} with the standard generalization bound for the offline setting). There is also a gap to be tightened in the other regime of $n$ and $d$ as well. For example, this can be seen by comparing \cite[Corollary~6.3]{BNS+16} with the bound attained by the private multiplicative weights algorithm \cite{hardt2010multiplicative} in the offline setting.

%In the centralized model, assuming $\eps = \Theta(1)$, the best known upper bound on the \emph{true} $L_{\infty}$ estimation error for this problem in the \emph{adaptive} setting is $\approx \log^{1/6}(J)\left(\frac{\log(d)}{n}\right)^{1/3}$ \cite[Corollary~6.3]{BNS+16} (which combines the bound attained by the private multiplicative weights algorithm \cite{hardt2010multiplicative} with the generalization guarantees of differential privacy). Whereas in the offline setting, the \emph{true} $L_{\infty}$ error is upper-bounded by $\approx \log^{1/4}(J)\sqrt{\frac{\log(d)}{n}}$ (combining \cite{hardt2010multiplicative} with the standard generalization error bound in the offline setting). %Hence, there 

%Hence, the error for the offline case is better than the best known bound for the adaptive case by a factor of $\approx \min\left(\frac{\sqrt{n}}{d^{1/4}}, ~\frac{d^{1/4}}{\sqrt{\log(d)}}\right)$ for all $d \lesssim n^2$. %This strongly suggests the possibility 
% is that Unlike what is currently known about this problem in the centralized model, our result also shows that in the local model there is no extra penalty in the \emph{true}$L_{\infty}$ error due to adaptivity in the case of linear queries.

%% file: prelim.tex
\section{Preliminaries and Definitions}\label{sec:prelim}

\ifnum\submission=1
A more detailed version of this section is provided in the attached full version.
\fi

\subsection{($\eps, \delta$)-Local Differential Privacy}

In the local model, an algorithm $\mathcal{A}$ can access any entry in a private data set $D=(v_1,\ldots,v_n)\in \mathcal{J}^n$ only via a randomized algorithm (local randomizer) $\mathcal{R}:\mathcal{J}\rightarrow \mathcal{W}$\ifnum\final =1 that, given an index $i\in [n],$ runs on the input $v_i$ and returns a randomized output $\mathcal{R}(v_i)$ to $\mathcal{A}$\fi. Such algorithm $\mathcal{A}$ satisfies $(\eps, \delta)$-local differential privacy ($(\eps, \delta)$-LDP) if the local randomizer $\cR$ satisfies $(\eps, \delta)$-LDP defined as follows. 

\begin{definition}[$(\eps, \delta)$-LDP]\label{defn:ldp}
	A randomized algorithm $\mathcal{R}:\mathcal{J}\rightarrow \mathcal{W}$ is $(\eps, \delta)$-LDP if for any pair $v, v' \in \cJ$ and any measurable subset $\cO\subseteq \mathcal{W},$ we have 
	 $$\pr{\cR}{\cR(v)\in\cO}\leq e^{\eps} \, \pr{\cR}{\cR(v)\in\cO} + \delta,$$
	 where the probability is taken over the random coins of $\cR$. The case of $\delta=0$ is called pure $\eps$-LDP.
\end{definition}

\subsection{Accuracy Definitions} 
\subsubsection{Offline queries}\label{sec:acc-def-offline}
For the non-adaptive (offline) setting, we measure accuracy in terms of the worst-case expected $L_2$-error in the responses to $d$ queries. \ifnum\final=1 Let $\bp$ be any (unknown) distribution over a data domain $\cJ=[J]$. To simplify presentation, we will overload notation and use $\bp\in \simp(J)$ to also denote the probability mass function (p.m.f.) of the same distribution, where $\simp(J)$ refers to the probability simplex in $\R^J$ defined as $\simp(J)=\left\{(w_1, \ldots, w_J)\in\R^J: ~w_j\geq 0 ~\forall j\in[J],~ \sum_{j=1}^Jw_j =1\right\}.$\else Let $\bp\in\simp(J)$ be any (unknown) distribution over a data domain $\cJ=[J]$, where $\simp(J)$ is the probability simplex in $\R^J$ defined as $\simp(J)=\left\{(w_1, \ldots, w_J)\in\R^J: ~w_j\geq 0 ~\forall j\in[J],~ \sum_{j=1}^Jw_j =1\right\}.$\fi
%\footnote{The term worst-case refers to maximum error over the data distribution and over the choice of the queries matrix $\bA$ subject to the aforementioned constraint.} 

Let $D$ denote the set of users' inputs $\{v_i: i\in[n]\}$ that are drawn i.i.d. from $\bp$\ifnum\final=1 (this will be usually denoted as $D\sim \bp^n$)\fi. For any $r>0$, let $\cC_{2}(r)=\left\{\bA = [\ba_1 ~\ldots~ \ba_J]\in \R^{d\times J}:~ \|\ba\|_2 \leq r\right\}$; that is, $\cC_{2}(r)$ denote the family of all matrices in $\R^{d\times J}$ whose columns lie in $B_2^d(r)$ (the $d$-dim $L_2$ ball of radius $r$). Let $\bA\in\cC_{2}(r)$ be a queries matrix whose rows determine $d$ offline linear queries. An $(\eps, \delta)$-LDP protocol $\prot$ describes a set of procedures executed at each user and the server that eventually produce an estimate $\hat\by\in\R^d$ for the true answer vector $\bA\bp \in \R^d$ subject to $(\eps, \delta)$-LDP. Let $\prot(\bA, D)$ denote the final estimate vector $\hat\by$ generated by the protocol $\prot$ for a data set $D$ and queries matrix $\bA$. The true expected $L_2$ error \ifnum\final=1 in the estimate $\prot(\bA, D)$ when $D\sim \bp^n$\fi is defined as \ifnum\final=1
$$\err_{\prot, L_2}(\bA; \bp^n)\triangleq\ex{\prot,~ D\sim\bp^n}{\|\prot(\bA, D)-\bA\bp\|_2},$$\else
$\err_{\prot, L_2}(\bA; \bp^n)\triangleq\ex{\prot,~ D\sim\bp^n}{\|\prot(\bA, D)-\bA\bp\|_2},$\fi
where the expectation is taken over the randomness in $D$ and the random coins of the protocol. 

\paragraph{True error:} The worst-case expected $L_2$-error (with respect to \emph{worst-case distribution} and \emph{worst case queries matrix} in $\cC_{2}(r)$) is defined as\ifnum\final=1 
\begin{align}
\err_{\prot, L_2}(\cC_2(r), n)&\triangleq\sup_{\bA\in\cC_2(r)}~\sup_{\bp\in \simp(J)}~\ex{\prot, ~D\sim\bp^n}{\|\prot(\bA; D)-\bA\bp\|_2}\label{l2-err}
\end{align}
\else $~\err_{\prot, L_2}(\cC_2(r), n)~\triangleq$
\begin{align}
&\sup_{\bA\in\cC_2(r)}~\sup_{\bp\in \simp(J)}~\ex{\prot, ~D\sim\bp^n}{\|\prot(\bA; D)-\bA\bp\|_2}\label{l2-err}
\end{align}
\fi

\paragraph{Empirical error:} Sometimes, we will consider the worst-case empirical $L_2$ error of an LDP protocol.  Given any data set $D\in[J]^n$, let $\hbp(D)\in\simp(J)$ denote the histogram (i.e., the empirical distribution) of $D$. The worst-case empirical $L_2$ error of an LDP protocol $\prot$ is defined as\ifnum\final=1  

\begin{align}
\emperr_{\prot, L_2}(\cC_2(r), n)&\triangleq\sup_{\bA\in\cC_2(r)}~\sup_{D\in [J]^n}~\ex{\prot}{\|\prot(\bA; D)-\bA\hbp(D)\|_2}\label{l2-emperr}
\end{align}
\else $~\emperr_{\prot, L_2}(\cC_2(r), n) ~\triangleq$
\begin{align}
&\sup_{\bA\in\cC_2(r)}~\sup_{D\in [J]^n}~\ex{\prot}{\|\prot(\bA; D)-\bA\hbp(D)\|_2}\label{l2-emperr}
\end{align}
\fi
Note the expectation in this case is taken only over the random coins of $\prot$.

\paragraph{Optimal non-private estimators for offline linear queries}%\label{sec:opt-non-priv}
The following is a simple observation that follows well-known facts in statistical estimation. %Let $\bA\in \cC_2(r)$ be any queries matrix and $\bp\in \simp(J)$ any distribution over $[J]$. Given a data set $D\sim \bp^n$, let $\hbp(D)\in\simp(J)$ denote the histogram (i.e., the empirical distribution) of $D$. Given $D$, The optimal non-private estimator of $\bA\bp$ under the expected $L_2$ error criterion is given by $\bA\,\hbp(D)$. In particular, we have
\begin{align}
\sup_{\bA\in\cC_2(r)}~\sup_{\bp\in \simp(J)}~\ex{D\sim\bp^n}{\|\bA\, \hbp(D)-\bA\bp\|_2}&\leq \frac{r}{\sqrt{n}}\label{ineq:opt-non-priv-bd}
\end{align}
\ifnum\final=1 Note that $\bA\, \hbp(D)$ is an unbiased estimator of $\bA\bp$. The above bound follows from a simple analysis of the variance of $\bA\, \hbp(D)$.\fi% and the constraint $\cC_2(r)$. 

\paragraph{Note:} Given (\ref{ineq:opt-non-priv-bd}), if we have an LDP protocol $\prot$ that %, for any queries matrix $\bA\in\cC_{2}(r)$ and any data set $D$, produces a private estimate $\hby$ that approximates the empirical non-private estimate $\bA\, \hbp(D)$ to within an $L_2$ error $\alpha$,
has worst-case \emph{empirical} $L_2$ error $\alpha$, then such a protocol has worst-case true $L_2$ error $\err_{\prot, L_2}(\cC_2(r), n)\leq \alpha + \frac{r}{\sqrt{n}}$.

%\paragraph{Note:} Given (\ref{ineq:opt-non-priv-bd}), if we have an LDP protocol $\prot$ that, for any queries matrix $\bA\in\cC_{2}(r)$ and any data set $D$, produces an estimate $\hby$ that approximates the empirical non-private estimate $\bA\, \hbp(D)$ to within an $L_2$ error $\alpha$, then such a protocol has worst-case expected $L_2$ error $\err_{\prot, L_2}(\cC_2(r), n)\leq \alpha + \frac{r}{\sqrt{n}}$.

\subsubsection{Adaptive queries}\label{sec:acc-def-adaptive}
\ifnum\final=1
For any $r>0$, we let $\cQ_{\infty}(r)=\left\{\langle \bq, \cdot \rangle:~ ~\|\bq\|_{\infty}\leq r\right\}$, i.e., $\cQ_{\infty}(r)$ denote the family of all linear queries described by vectors in $\R^J$ of $L_{\infty}$ norm bounded by $r$. In the adaptive setting, we consider the worst-case expected $L_{\infty}$ error in the vector of estimates generated by LDP protocol for any sequence of $d$ adaptively chosen queries $\bq_1, \ldots, \bq_d \in\cQ_{\infty}$. Let $D\sim \bp^n$ be a data set of users' inputs. Let $\prot$ be LDP protocol for answering any such sequence. We define the worst-case $L_{\infty}$ error as
\begin{align}
\err_{\prot, L_{\infty}}(\cQ_{\infty}(r), d, n)&\triangleq\sup\limits_{\bp\in\simp(J)}~~ \sup\limits_{\substack{\text{adaptive strategy}\\ \text{choosing }~\bq_1, \ldots, \bq_d}}~ \ex{\prot, ~D\sim\bp^n}{\max\limits_{k\in [d]}~\lvert~ \prot^{(k)}(D) - \langle \bq_k, ~\bp\rangle ~\rvert~},\label{linf-err}
\end{align}
where $\prot^{(k)}(D)$ denotes the estimate generated by the protocol in the $k$-th round of the protocol.

\else
For any $r>0$, we let $\cQ_{\infty}(r)=\left\{\langle \bq, \cdot \rangle:~ ~\|\bq\|_{\infty}\leq r\right\}$, i.e., $\cQ_{\infty}(r)$ denote the family of all linear queries described by vectors in $\R^J$ of $L_{\infty}$ norm bounded by $r$. In the adaptive setting, we consider the worst-case expected $L_{\infty}$ error in the vector of estimates generated by a LDP protocol for any sequence of $d$ adaptively chosen queries $\bq_1, \ldots, \bq_d \in\cQ_{\infty}$. %Let $D\sim \bp^n$ be a data set of users' inputs. Let $\prot$ be LDP protocol for answering any such sequence. 
We define the worst-case $L_{\infty}$ error of a protocol $\prot$ as $\err_{\prot, L_{\infty}}(\cQ_{\infty}(r), d, n)~\triangleq$
\begin{align}
\hspace{-0.3cm}&\sup\limits_{\bp\in\simp(J)} \sup\limits_{\substack{\text{~adaptive strategy}\\ \text{for  }~\bq_1, \ldots, \bq_d}}~ \ex{\prot, ~D\sim\bp^n}{\max\limits_{k\in [d]}~\lvert~ \prot^{(k)}(D) - \langle \bq_k, ~\bp\rangle ~\rvert~},\label{linf-err}
\end{align}
where $\prot^{(k)}(D)$ denotes the estimate generated by the protocol in the $k$-th round of the protocol.
\fi

\subsection{Geometry facts}

\ifnum\final=1

For a convex body  $K\subseteq \R^d$, the polar body $K_o$ is defined as $\{\by: \vert\langle \by, \bx\rangle \vert\leq 1~\forall \bx\in K\}$. A convex body $K$ is symmetric if $K=-K$. The Minkowski norm $\|\bx\|_K$ induced by a symmetric convex body $K$ is defined as $\|\bx\|_K=\inf\{r\in\R: \bx\in rK\}$. The Minkowski norm induced by the polar body $K_o$ of $K$ is the dual norm of $\|\bx\|_K$, and has the form $\|\by\|_{K_o}=\sup_{\bx\in K}\lvert \langle \bx, \by\rangle\rvert$. By Holder's inequality, we have $\langle \bx, \by\rangle\leq \|\bx\|_K\|\by\|_{K_o}$.
\fi 

Let $\B_1^J$ denote the unit $L_1$ ball in $\R^J$. A symmetric convex polytope $L\subset \R^d$ of $J$ vertices that are represented as the columns of a matrix $\bA\in\R^{d\times J}$ is defined as $L\triangleq\bA\B_1^J=\{\by\in\R^d: \by= \bA\bx~\text{for some }\bx\in\R^J~\text{with }\|\bx\|_1\leq 1\}.$ \ifnum\final=1 The \emph{dual} Minkowski norm induced by the convex symmetric polytope $L$ is given by $\|\bx\|_{L_o}=\max_{\by\in L}\lvert\langle\bx, \by\rangle\rvert=\max_{j\in[J]}\lvert\langle\ba_j, \bx\rangle\rvert,$ where the last equality is due to the fact that any linear function over a polytope attains its maximum at one of the vertices of the polytope. 

The following is a useful lemma based on standard analysis that bounds the least squared estimation error over convex bodies. We restate here the version that appeared in \cite{NTZ12}. 

\begin{lem}[Lemma 1 in \cite{NTZ12}]\label{proj_lemma}
Let $L\subseteq \R^d$ be a symmetric convex body, and let $\by\in L$ and $\bar{\by}=\by+\bz$ for some $\bz\in\R^d$. Let $\hat\by=\arg\min_{\bw\in L}\|\bw-\bar{\by}\|_2^2$. Then, we must have 
$$\|\hat\by-\by\|^2_2\leq 4\min\{\|\bz\|_2^2,~\|\bz\|_{L_o}\}.$$
\end{lem}

As a direct consequence of the above lemma and the preceding facts, we have the following corollary. 
\begin{corollary}\label{proj_cor}
Let $L\subset \R^d$ be a symmetric convex polytope of $J$ vertices $\{\ba_j\}_{j=1}^J$, and let $\by\in L$ and $\bar{\by}=\by+\bz$ for some $\bz\in\R^d$. Let $\hat\by=\arg\min_{\bw\in L}\|\bw-\bar{\by}\|_2^2$. Then, we must have 
$$\|\hat\by-\by\|^2_2\leq 4\max_{j\in[J]}\lvert\langle\bz,~\ba_j\rangle\rvert.$$
\end{corollary}
\else

The following fact is a direct consequence of standard results in convex geometry (see the attached full version for details).
\begin{fact}\label{proj_cor}
	Let $L\subset \R^d$ be a symmetric convex polytope of $J$ vertices $\{\ba_j\}_{j=1}^J$, and let $\by\in L$ and $\bar{\by}=\by+\bz$ for some $\bz\in\R^d$. Let $\hat\by=\arg\min_{\bw\in L}\|\bw-\bar{\by}\|_2^2$. Then, we must have 
	$$\|\hat\by-\by\|^2_2\leq 4\max_{j\in[J]}\lvert\langle\bz,~\ba_j\rangle\rvert.$$
\end{fact}
\fi
\ifnum\final=1
\subsection{SubGaussian random variables}
\begin{defn}[$\sigma^2$-subGaussian random variable]
	A zero mean random variable $X$ is called $\sigma$-subgaussian if for all $\lambda\geq 0,$ ~$\pr{}{\lvert X\rvert \geq \lambda}\leq 2\,e^{-\frac{\lambda^2}{2\sigma^2}}.$
	\end{defn}

Another equivalent version of the definition is as follows: A zero-mean random variable $X$ is $\sigma$-subgaussian if for all $t\in \R, ~\ex{}{e^{t\,X}}\leq e^{\frac{1}{2}t^2\,\sigma^2}$. It is worth noting that these two versions of the definition are equivalent up to a small constant in $\sigma$ (see, e.g., \cite{buldygin2000metric}). 

\fi
%
%The following theorems state useful properties of subgaussian random variables (see \cite{buldygin2000metric} for details). 
%
%\begin{thm}\label{thm:subgauss-var}
%	If $X$ is $\sigma$-subgaussian random variable, then $\var{}{X}\leq 4\sigma^2$. 
%\end{thm}
%
%\begin{thm}\label{thm:sum-subgauss}
%	If $X_1, \ldots, X_k$ are independent random variables such that $X_j$ is $\sigma_j$-subgaussian for each $j\in [k]$. Then, for any $a_1, \ldots, a_k \in \R$,~ $\sum_{j=1}^ka_j\,X_j$ is $\sqrt{\sum_{j=1}^k a_j^2\sigma_j^2}$-subgaussian. 
%\end{thm}
%

%% file: Offline.tex
\section{LDP Protocols for Offline Linear Queries}\label{sec:L2_ball_queries}

In this section, we consider the problem of estimating $d$ offline linear queries under $\eps$-LDP. For any given $r>0$, as discussed in Section~\ref{sec:acc-def-offline}, we consider a queries matrix $\bA\in \cC_2(r)$\ifnum\final=1; that is, the columns of $\bA$ are assumed to lie in the $L_2$ ball $\B_2^d(r)$ of radius $r$\fi. %Without loss of generality, we will assume $r=1$. All bounds in this section can be translated to the general case by scaling them by $r$. We refer to the unit $L_2$ ball in $\R^d$ as $\B_2^d$. 
%We provide nearly tight upper bound on the $L_2$ error for this problem.

%We provide constructions that work  $L_2$ error for this problem.

\ifnum\final=1

As a warm-up, in Section~\ref{subsec:gauss_L2_ball}, we first describe and analyze an $(\epsilon, \delta)$-LDP protocol. Our protocol is simple and is based on (i) perturbing the columns of $\bA$  corresponding to users' inputs via Gaussian noise and (ii) applying a projection step, when appropriate, to the noisy aggregate similar to the technique of \cite{NTZ12} in the centralized model. This projection step reduces the error significantly in the regime where $n\lesssim d^2/\log(J)$ (which subsumes the high-dimensional setting $d\gtrsim n$ when $\log(J)\lesssim d$). In particular, in such regime, our protocol yields an $L_2$ error $\approx r\,\left(\frac{\log(J)}{n}\right)^{1/4}$, which does not depend on $d$ and depends only sub-logarithmically on $J$. Moreover, this error is within a factor of $\log^{1/4}(J)$ from the optimal error in this regime. Hence, this result establishes the possibility of accurate estimation of linear queries with respect to the $L_2$ error in high-dimensional settings. Adoption of all previously known algorithms (particularly, the ball sampling mechanism of \cite{DJW13}) do not provide any guarantees better than the trivial error for that problem in the regime where $n\lesssim d$. 

\else

As a warm-up, in Section~\ref{subsec:gauss_L2_ball}, we first describe and analyze an $(\epsilon, \delta)$-LDP protocol. Our protocol is simple and is based on (i) perturbing the columns of $\bA$  corresponding to users' inputs via Gaussian noise and (ii) applying a projection step, when appropriate, to the noisy aggregate similar to the technique of \cite{NTZ12} in the centralized model. This projection step reduces the error significantly in the regime where $n\lesssim d^2/\log(J)$. In particular, in such regime, our protocol yields an $L_2$ error $\approx r\,\left(\frac{\log(J)}{n}\right)^{1/4}$, which is within a factor of $\log^{1/4}(J)$ from the optimal error. Adoption of all previously known algorithms (particularly, the ball sampling mechanism of \cite{DJW13}) do not provide any guarantees better than the trivial error for that problem in the regime where $n\lesssim d$. 
\fi

In Section~\ref{subsec:ball_samp}, we give a construction that transforms our $(\eps, \delta)$ algorithm into a pure $\epsilon$-LDP algorithm with essentially the same error guarantees. Our transformation is inspired by ideas from \cite{BS15,BNS18}. In particular, \cite{BNS18} gives a generic technique for transforming an $(\eps, \delta)$-LDP protocol to an $O(\eps)$-LDP protocol. Our construction can be viewed as a simpler, more direct version of this transformation for the case of linear queries. 

%Then, we prove a tight lower bound on the error of pure $\eps$-locally private algorithms for answering linear queries in Section~\ref{subsec:lower_bd}. Our results improve over the bounds in \cite{DJW13}. In particular, we tighten both the upper and lower bounds implied by \cite[Propositions 3 and 4]{DJW13}.

\subsection{$(\epsilon, \delta)$ LDP Protocol for Offline Linear Queries}\label{subsec:gauss_L2_ball}

We first describe the local randomization procedure $\cRG_i$ carried out by each user $i\in [n]$. The local randomization is based on perturbation via Gaussian noise \ifnum\final=1; that is, it can be viewed as LDP version of the standard Gaussian mechanism \cite{DKMMN06}\fi. 

\begin{algorithm}[htb]
	\caption{$\cRG_i$: $(\eps, \delta)$-Local Randomization of user $i\in[n]$}
	\begin{algorithmic}[1]
		\REQUIRE Queries matrix $\bA\in\cC_{2}(r),$ User $i$ input $v_i\in [J]$, privacy parameters $\eps, \delta$.
	\RETURN $\tilde{\by}_{i}=\ba_{v_i}+\bz_{i}$ where $z_{i}\sim\cN(\mathbf{0}, \sigma^2\mathbb{I}_d)$ where $\ba_{v_i}$ is the $v_i$-th column of $\bA$, $\sigma^2=2\,r^2\,\frac{\log(2/\delta)}{\epsilon^2}$, and $\mathbb{I}_d$ denotes the identity matrix of size $d$.\label{Q-pp-2}
	\end{algorithmic}
	\label{Alg:locrnd-gauss}
\end{algorithm}

The desciption of our $(\eps, \delta)$ protocol for linear queries is given in Algorithm~\ref{Alg:gauss-ball}.

\begin{algorithm}[htb]
	\caption{$\prot_{\gauss}$: $(\eps, \delta)$-LDP protocol for answering offline linear queries from $\cC_{2}(r)$}
	\begin{algorithmic}[1]
		\REQUIRE Queries matrix $\bA\in\cC_{2}(r),$ Users' inputs $\{v_i\in [J]: i\in[n]\}$, privacy parameters $\eps, \delta$.
        \FOR{ Users $i=1$ to $n$}
	   {\STATE User $i$ computes $\tby_i= \cRG_i(v_i)$ and sends it to the server.}
        \ENDFOR
	  %{\STATE Server receives $\{\tilde{\by}_{i}\}_{i=1}^n$.}
{\STATE Server computes $\bar{\by}=\frac{1}{n}\sum_{i=1}^n\tilde{\by}_i$.}
	\IF {$n < \frac{d^2\log(2/\delta)}{8\,\epsilon^2\,\log(J)}$:} 
		{\STATE $\hat{\by}=\arg\min_{\bw\in \bA\B_1^J}\|\bw-\bar{\by}\|_2^2$ where $\B_1^J$ is the unit $L_1$ ball in $\R^J$.}
	\ELSE {\STATE $\hat{\by}=\bar{\by}$}
	\ENDIF
	\RETURN $\hat{\by}$.
	\end{algorithmic}
	\label{Alg:gauss-ball}
\end{algorithm}

\ifnum\final=1
We now state and prove the privacy and accuracy guarantee of our protocol. Note in the local model of differential privacy, the privacy of the entire protocol rests only on differential privacy of the local randomizers, which we prove now. 

\else

We now give the privacy and accuracy guarantees of our protocol. 

\fi

\begin{thm}\ifnum\final=1 [Privacy Guarantee]\fi \label{thm:privacy_gauss}
Algorithm~\ref{Alg:locrnd-gauss} is $(\epsilon, \delta)$-LDP.
\end{thm}
\ifnum\final=1
\begin{proof}
The proof follows directly from standard analysis of the Gaussian mechanism \cite{DKMMN06, NTZ12} applied in the context of $(\epsilon, \delta)$- LDP. 
\end{proof}
\else
The proof follows directly from standard analysis of the Gaussian mechanism \cite{DKMMN06, NTZ12}.
\fi

\ifnum\final=1

\begin{thm}[Accuracy of Algorithm~\ref{Alg:gauss-ball}]\label{thm:accuracy_gauss}
 Protocol $\prot_{\gauss}$ given by Algorithm~\ref{Alg:gauss-ball} satisfies the following accuracy guarantee:
$$\err_{\prot_{\gauss}, ~L_2}(\cC_2(r), n)\leq ~r\cdot \min\left(\left(\frac{32\,\log(J)\log(2/\delta)}{n\epsilon^2}\right)^{1/4},~ \sqrt{\frac{2\,d\log(2/\delta)}{n\epsilon^2}}\right)$$
where $\err_{\prot_{\gauss}, ~L_2}(\cC_2(r), n)$ is as defined in (\ref{l2-err}).
\end{thm}

\else

\begin{thm}\label{thm:accuracy_gauss}
	The worst-case $L_2$ error of Algorithm~\ref{Alg:gauss-ball} $\err_{\prot_{\gauss}, ~L_2}(\cC_2(r), n)$ is upper-bounded by
	$$ r\cdot \min\left(\left(\frac{32\,\log(J)\log(2/\delta)}{n\epsilon^2}\right)^{1/4},~ \sqrt{\frac{2\,d\log(2/\delta)}{n\epsilon^2}}\right),$$
	where $\err_{\prot_{\gauss}, ~L_2}(\cC_2(r), n)$ is as defined in (\ref{l2-err}).
\end{thm}
\fi

\begin{proof}
Fix any queries matrix $\bA\in\cC_2(r)$. Let $\by=\bA\hbp$ where $\hbp=\frac{1}{n}\sum_{i=1}^n \be_{v_i}$ is the actual histogram of the users' data set (here, $\be_t\in\R^J$ denotes the vector with $1$ in the $t$-th coordinate and zeros elsewhere). First, consider the case where $n\geq \frac{d^2\log(2/\delta)}{8\,\epsilon^2\,\log(J)}$. Note that $\hat{\by}=\bar{\by}$, and hence $\hat{\by}-\by$ is Gaussian random vector with zero mean and covariance matrix $\frac{\sigma^2}{n}\mathbb{I}_d$. Hence, in this case, it directly follows that $\emperr_{\prot_{\gauss}, ~L_2}(\cC_2(r), n)= \sqrt{\frac{\sigma^2\,d}{n}}= r\,\sqrt{\frac{2\,d\log(2/\delta)}{n\epsilon^2}},$ where $\emperr_{\prot_{\gauss}, ~L_2}(\cC_2(r), n)$ is the worst-case empirical error as defined in (\ref{l2-emperr}).

Next, consider the case where $n< \frac{d^2\log(2/\delta)}{8\,\epsilon^2\,\log(J)}$. Since $\hat{\by}$ is the projection of $\bar{\by}$ on the symmetric convex polytope $\bA\B_1^J$, then by \ifnum\final=1 Corollary~\ref{proj_cor}\else Fact~\ref{proj_cor}\fi, it follows that 
$$\|\hat{\by}-\by\|_2^2\leq 4\max\limits_{j\in [J]}\lvert\langle \bar{\by}-\by,~\ba_j\rangle\rvert.$$
Hence, %by Cauchy-Schwartz inequality (or, by Jensen's inequality), 
we have
$$\emperr_{\prot_{\gauss}, ~L_2}(\cC_2(r), n)\leq 2\sqrt{\ex{}{\max\limits_{j\in [J]}\lvert\langle \bar{\by}-\by,~\ba_j\rangle\rvert}}.$$
As before, note that $\bar{\by}-\by ~ \sim \cN\left(\bzero, ~\frac{\sigma^2}{n}\mathbb{I}_d\right)$. Note also that $\|\ba_j\|\leq r~ ~\forall j\in [J]$. Hence, for each $j\in [J]$, $\langle \bar{\by}-\by,~\ba_j\rangle$ is Gaussian with zero mean and variance  $\leq r^2\,\sigma^2/n$. By standard bounds on the maximum of Gaussian r.v.s  (e.g., see \cite{mit}), we have 
\ifnum\final=1 
$$\ex{}{\max\limits_{j\in [J]}\lvert\langle \bar{\by}-\by,~\ba_j\rangle\rvert} \leq \sqrt{\frac{\sigma^2}{n}r^2\log(J)}\leq r^2\,\sqrt{2\frac{\log(J)\log(2/\delta)}{n\epsilon^2}}.$$

Hence, in this case, we have ~$\emperr_{\prot_{\gauss}, ~L_2}(\cC_2(r), n)\leq \left(\frac{32\,\log(J)\log(2/\delta)}{n\epsilon^2}\right)^{1/4}$. 

Putting the two cases above together, we get that  $\emperr_{\prot_{\gauss}, ~L_2}(\cC_2(r), n)$ is upper-bounded by the expression in the theorem statement.

From (\ref{ineq:opt-non-priv-bd}) in Section~\ref{sec:acc-def-offline} (and the succeeding note), we have 
$$\err_{\prot_{\gauss}, ~L_2}(\cC_2(r), n)\leq \emperr_{\prot_{\gauss}, ~L_2}(\cC_2(r), n) + r/\sqrt{n}.$$
Note that the $r/\sqrt{n}$ term above is swamped by the bound on $\emperr_{\prot_{\gauss}, ~L_2}(\cC_2(r), n)$. This completes the proof.

\vspace{0.3cm}

\else
$$\ex{}{\max\limits_{j\in [J]}\lvert\langle \bar{\by}-\by,~\ba_j\rangle\rvert} \leq r^2\,\sqrt{2\frac{\log(J)\log(2/\delta)}{n\epsilon^2}}.$$
Hence, in this case, we have 
$$\emperr_{\prot_{\gauss}, ~L_2}(\cC_2(r), n)\leq \left(\frac{32\,\log(J)\log(2/\delta)}{n\epsilon^2}\right)^{1/4}.$$

Putting the two cases above together, we get that  $\emperr_{\prot_{\gauss}, ~L_2}(\cC_2(r), n)$ is upper-bounded by the expression in the theorem statement. From (\ref{ineq:opt-non-priv-bd}) in Section~\ref{sec:acc-def-offline} (and the succeeding note), we have 
$$\err_{\prot_{\gauss}, ~L_2}(\cC_2(r), n)\leq \emperr_{\prot_{\gauss}, ~L_2}(\cC_2(r), n) + r/\sqrt{n}.$$
Note that the $r/\sqrt{n}$ term above is swamped by the bound on $\emperr_{\prot_{\gauss}, ~L_2}(\cC_2(r), n)$. 
\fi

\end{proof}

\subsection{$(\epsilon, 0)$ LDP Protocol for Offline Linear Queries}\label{subsec:ball_samp}

\ifnum\final=1
In this section, we give a pure LDP construction that achieves essentially the same accuracy (up to a constant factor of at most $2$) as our approximate LDP algorithm above. Our construction is based on a direct transformation of the above approximate LDP protocol into a pure LDP one. Our construction is inspired by the idea of rejection sampling in \cite{BS15, BNS18}, and can be viewed as a simpler, more direct version of the generic technique in \cite{BNS18} in the case of linear queries. 

\else

In this section, we give a pure LDP construction that achieves essentially the same accuracy (up to a constant factor of at most $2$) as our approximate LDP algorithm above. Our construction is based on a direct transformation of the above protocol into a pure LDP one. Our construction is inspired by the idea of rejection sampling in \cite{BS15, BNS18}. 

\fi

In our construction, we assume that $\eps \leq 1$\footnote{This is not a loss of generality in most practical scenarios where we aim at a reasonably strong privacy guarantee.}. For any $\ba\in \R^d$, let $f_{\ba}$ denote the probability density function of the Gaussian distribution $\cN(\ba, \sigma^2\,\iden_d)$ where $\sigma^2= 4\,r^2\,\frac{\log(n)}{\epsilon^2}$. (Note that the setting of $\sigma^2$ is the same setting for the Gaussian noise used in Algorithm~\ref{Alg:gauss-ball} with $\delta\approx 1/n^2$.)

In Algorithm~\ref{Alg:locrnd}, we describe the local randomization procedure $\cRS_i$ executed independently by every user $i\in [n]$. Then, we describe our $\eps$-LDP protocol for offline linear queries in Algorithm~\ref{Alg:samp-ball}.

\begin{algorithm}[htb]
	\caption{$\cRS_i$: $\eps$-Local Randomization of user $i\in[n]$ based on rejection sampling}
	\begin{algorithmic}[1]
		\REQUIRE Queries matrix $\bA\in\cC_{2}(r),$ User $i$ input $v_i\in [J]$, privacy parameter $\eps$.
		\STATE Get $\ba_{v_i}$: ~the $v_i$-th column of $\bA$.		
		\STATE Sample a Gaussian vector $\tby_i \sim \cN\left(\mathbf{0},~\sigma^2\,\iden_d\right)$, where \ifnum\final=1 $\sigma^2:= 2\,r^2\,\frac{\log(2/\delta)}{\epsilon^2}$ and $\delta:= \frac{2}{n^2}$. \else $\sigma^2:= 4\,r^2\,\frac{\log(n)}{\epsilon^2}$. \fi \label{step:delta-set}
	\STATE Compute (scaled) ratio of the two Gaussian densities $f_{\ba_{v_i}}$ and $f_{\bzero}~$ at $~\tby_i: ~~\eta_i:=\frac{1}{2}\,\frac{f_{\ba_{v_i}}(\tby_i)}{f_{\bzero}(\tby_i)}.$ \label{step:ratio}
		\IF{  $\eta_i \in [\frac{e^{-\eps/4}}{2}, ~\frac{e^{\eps/4}}{2}]$} 
		{\STATE Sample a bit $B_i\sim \Ber(\eta_i)$}
		\ELSE {\STATE Let $B_i =0$}
		\ENDIF
		\IF{ $B_i=1$} {\RETURN $\tby_i$}
		\ELSE {\RETURN $\bot$} \COMMENT {The output in this case indicates that user $i$ is dropped out of the protocol.}
		\ENDIF
	\end{algorithmic}
\label{Alg:locrnd}
\end{algorithm}

%The description of our $\eps$-LDP protocol for answering linear queries is given by Algorithm~\ref{Alg:samp-ball}.  

\begin{algorithm}[htb]
	\caption{$\prot_{\samp}$: $\eps$-LDP protocol for offline linear queries from $\cC_{2}(r)$}
	\begin{algorithmic}[1]
		\REQUIRE Queries matrix $\bA\in\cC_{2}(r),$ Users' inputs $\{v_i\in [J]: i\in[n]\}$, privacy parameter $\eps$.

\FOR{  All users $i\in [n]$ \textbf{such that} $\cRS_i(v_i)\neq \bot$}
	   {\STATE Let $\tby_i= \cRS_i(v_i)$ and send $\tby_i$ to the server. \label{step:local_rand}}
        \ENDFOR
{\STATE Server receives the set of responses $\{\tilde{\by}_{i}\}_{i=1}^{\hn}$, where $\hn$ is the number users whose response $\neq\bot$.\label{step:aggreg}}

{\STATE Server computes $\bar{\by}=\frac{1}{\hn}\sum_{i=1}^{\hn}\tilde{\by}_i$.}
\IF {$\hn < \frac{d^2\log(n)}{4\,\epsilon^2\,\log(J)}$:} 
{\STATE $\hat{\by}=\arg\min_{\bw\in \bA\B_1^J}\|\bw-\bar{\by}\|_2^2$ where $\B_1^J$ is the unit $L_1$ ball in $\R^J$.}
\ELSE {\STATE $\hat{\by}=\bar{\by}$}
\ENDIF
\RETURN $\hat{\by}$.
	\end{algorithmic}
	\label{Alg:samp-ball}
\end{algorithm}

\ifnum\final=1
We now  state and prove  the privacy and accuracy guarantees of our protocol. 
\else
The privacy and accuracy guarantees are given by the following theorems.
\fi

\begin{thm}\ifnum\final=1[Privacy Guarantee]\fi\label{thm:privacy_samp}
	Algorithm~\ref{Alg:locrnd} is $\eps$-LDP.
\end{thm}

\ifnum\final=1
\begin{proof}
 Consider any user $i\in [n]$. Let $v\in [J]$ be any input of user $i$. Define $$\gd_i(v)\triangleq\left\{\by\in\R^d:~ \eta_i(v, \by) \in [\frac{e^{-\eps/4}}{2},~\frac{e^{\eps/4}}{2}]\right\},$$ where $\eta_i(v,~\by)=\frac{1}{2}\,\frac{f_{\ba_{v}}(\by)}{f_{\bzero}(\by)}$. Note that by the standard analysis of the Gaussian mechanism, we have $\pr{\tby_i \sim \cN\left(\mathbf{0},~\sigma^2\right)}{\tby_i\notin\gd_i(v)}\leq \delta$, where $\sigma^2$ and $\delta$ are set as in Step~\ref{step:delta-set} of Algorithm~\ref{Alg:locrnd}). Now, we note that the output of Algorithm~\ref{Alg:locrnd} is a function of only the bit $B_i$. Since differential privacy is resilient to post-processing, it suffices to show that for any $v, v'\in [J]$, any $b\in\{0, 1\}$, we have $\pr{\cRS_i(v)}{B_i=b}\leq e^{\epsilon}\pr{\cRS_i(v')}{B_i=b}$. First, observe that 
	\begin{align*}
		\pr{\cRS_i(v)}{B_i=0}&\leq\pr{\cRS_i(v)}{B_i=0\vert ~\tby_i\in\gd_i(v)}+ \pr{\cRS_i(v)}{\tby_i\notin\gd_i(v)}\\
		&\leq 1-\frac{e^{-\eps/4}}{2} + \delta.
	\end{align*}
	We also have 
	\begin{align*}
	\pr{\cRS_i(v')}{B_i=0}&\geq\pr{\cRS_i(v')}{B_i=0\vert ~\tby_i\in\gd_i(v')}\cdot \pr{\cRS_i(v')}{\tby_i\in\gd_i(v')}\\
	&\geq \left(1-\frac{e^{\eps/4}}{2}\right)\left(1-\delta\right).
	\end{align*}
	Thus, $\frac{\pr{\cRS_i(v)}{B_i=0}}{\pr{\cRS_i(v')}{B_i=0}}\leq \frac{1-\frac{e^{-\eps/4}}{2} + \delta}{\left(1-\frac{e^{\eps/4}}{2}\right)\left(1-\delta\right)}$. Note that for any $t\in \R$, $1+t\leq e^{t}$. Also, note that since $\eps \leq 1$, we have $1+\eps/4 \leq e^{\eps/4} \leq 1+ \frac{5}{16}\eps$. Hence, this ratio can be upper bounded as 
	\begin{align*}
		\frac{\frac{1}{2}(1+\eps/4)+\delta}{\frac{1}{2}(1-\frac{5}{16}\eps)(1-\delta)}&=\frac{1+\eps/4}{1-\frac{5}{16}\eps}\cdot \frac{1+\frac{2\delta}{(1+4\eps)}}{1-\delta}\leq e^{\frac{7}{8}\eps}e^{4\delta}\leq e^{\eps}.
	\end{align*}
	In the last step, we use the fact that $\delta = (1/n^2)$ and hence, $\delta \ll \eps/32$. 
	
	Now, we consider the event that $B_i=1$. Note that $\forall v\in [J],$ $\pr{\cRS_i(v)}{B_i=1~\vert~\tby_i\notin \gd_i(v)}=0$. Hence, we have 
	$$\pr{\cRS_i(v)}{B_i=1} \leq \pr{\cRS_i(v)}{B_i=1~\vert~\tby_i\in \gd_i(v)}\leq \frac{e^{\eps/4}}{2}.$$
	
	We also have 
	\begin{align*}
			\pr{\cRS_i(v')}{B_i=1}&=\pr{\cRS_i(v')}{B_i=1~\vert~\tby_i\in \gd_i(v')}\cdot \pr{\cRS_i(v')}{\tby_i\in\gd_i(v')} \\
			&\geq \frac{e^{-\eps/4}}{2} (1-\delta).
	\end{align*}
 
	Hence, 
	$$\frac{\pr{\cRS_i(v)}{B_i=1}}{\pr{\cRS_i(v')}{B_i=1} }\leq e^{\eps/2 + 2\delta}< e^{\eps}.$$

\end{proof}

\else
A detailed proof is provided in the attached full version. We give here a high-level idea of the proof technique. To show that $\cR_i$ is $\eps$-LDP, it suffices to show that for any $v, v'\in [J]$, any $b\in\{0, 1\}$, we have $\pr{\cRS_i(v)}{B_i=b}\leq e^{\epsilon}\pr{\cRS_i(v')}{B_i=b}$. We start by defining ``good''  event $\gd_i(v)\triangleq\left\{\by\in\R^d:~ \frac{1}{2}\,\frac{f_{\ba_{v}}(\by)}{f_{\bzero}(\by)} \in [\frac{e^{-\eps/4}}{2},~\frac{e^{\eps/4}}{2}]\right\}$. Note that by the standard analysis of the Gaussian mechanism, we have $\pr{}{\tby_i\notin\gd_i(v)}\lesssim 1/n^2$, when $\tby_i \sim \cN\left(\mathbf{0},~\sigma^2\right)$. We then proceed to obtain upper and lower bounds on $\pr{}{B_i=b}$ in terms of $\pr{}{B_i=b\vert ~\tby_i\in\gd_i(v)}$ and $\pr{}{\tby_i\notin\gd_i(v)}$. Since the former always lies in $[\frac{e^{-\eps/4}}{2},~\frac{e^{\eps/4}}{2}]$ and the latter is bounded by $1/n^2$, we can bound the aforementioned ratio by $e^{\eps}$.

\fi

\ifnum\final=1
\begin{thm}[Accuracy of Algorithm~\ref{Alg:samp-ball}]\label{thm:accuracy_samp}
	Suppose $n\geq 120$. Then, Protocol $\prot_{\samp}$ (Algorithm~\ref{Alg:samp-ball}) satisfies the following accuracy guarantee:
	$$\err_{\prot_{\samp}, ~L_2}(\cC_2(r), n)\leq ~r\cdot \min\left(\left(\frac{280\,\log(J)\log(n)}{n\epsilon^2}\right)^{1/4},~ \sqrt{\frac{10\,d\log(n)}{n\epsilon^2}}\right)$$
	where $\err_{\prot_{\samp}, ~L_2}(\cC_2(r), n)$ is as defined in (\ref{l2-err}).
\end{thm}
\else

\begin{thm}\label{thm:accuracy_samp}
	Suppose $n\geq 120$. Then, Protocol $\prot_{\samp}$ (Algorithm~\ref{Alg:samp-ball}) has $L_2$ error $\err_{\prot_{\samp}, ~L_2}(\cC_2(r), n)$ that is upper bounded as
	$$r\cdot \min\left(\left(\frac{280\,\log(J)\log(n)}{n\epsilon^2}\right)^{1/4},~ \sqrt{\frac{10\,d\log(n)}{n\epsilon^2}}\right)$$
	where $\err_{\prot_{\samp}, ~L_2}(\cC_2(r), n)$ is as defined in (\ref{l2-err}).
\end{thm}

\fi

\ifnum\final=1
The high-level idea of the proof can be described as follows. We first show that the number of users who end up sending a signal to the server (i.e., those users with $B_i = 1$) is at least a constant fraction of the total number of users ($\gtrsim n/4$). Hence, the effective reduction in the sample size will not have a pronounced effect on the true error (it can only increase the true expected $L_2$ error by at most a factor $\leq 2$). Next, we show that \emph{conditioned on} $B_i=1$, the distribution of the user's signal $\tby_i$ in Algorithm~\ref{Alg:samp-ball} is identical to the distribution of the user's signal in the $(\eps, \delta)$ protocol of the previous section (Algorithm~\ref{Alg:gauss-ball}). That is, \emph{conditioned on a high probability event}, the signals generated by the active users via the pure $\eps$ local randomizers $\cRS$ (Algorithm~\ref{Alg:locrnd}) are statistically indistinguishable from the signals that could have been generated if those users have used the Gaussian local randomizers $\cRG_i$ (Algorithm~\ref{Alg:locrnd-gauss}). This allows us to show that the $L_2$ error resulting from Algorithm~\ref{Alg:samp-ball} is essentially the same as the one resulting from Algorithm~\ref{Alg:gauss-ball}. 

\else

The detailed proof is provided in the attached full version. The high-level idea of the proof can be described as follows. We first show that the number of users who end up sending a signal to the server (i.e., those users with $B_i = 1$) is at least a constant fraction of the total number of users ($\gtrsim n/4$). Hence, the effective reduction in the sample size will not have a pronounced effect on the true error (it can only increase the true expected $L_2$ error by at most a factor $\leq 2$). Next, we show that \emph{conditioned on} $B_i=1$, %the distribution of the user's signal $\tby_i$ in Algorithm~\ref{Alg:samp-ball} is identical to the distribution of the user's signal in the $(\eps, \delta)$ protocol of the previous section (Algorithm~\ref{Alg:gauss-ball}). That is, 
the signal generated by an active user via the pure $\eps$ local randomizer $\cRS_i$ %(Algorithm~\ref{Alg:locrnd})
is identically distributed to (hence, statistically indistinguishable from) the signal that could have been generated if this users has used the Gaussian local randomizers $\cRG_i$ (Algorithm~\ref{Alg:locrnd-gauss}). This allows us to show that the $L_2$ error resulting from Algorithm~\ref{Alg:samp-ball} is essentially the same as the one resulting from Algorithm~\ref{Alg:gauss-ball}. 
\fi

\ifnum\final=1
We now give the formal proof. In the sequel, we call user $i\in [n]$ \emph{active} if $B_i=1$; that is, if $\cRS_i(v_i)\neq \bot$ and hence, user $i$ ends up sending a signal $\tby_i$ to the server.  As in the proof of Theorem~\ref{thm:privacy_samp}, we define 
$$\gd_i=\left\{\by\in\R^d:~ \eta_i(\by) \in \left[\frac{e^{-\eps/4}}{2}, ~\frac{e^{\eps/4}}{2}\right]\right\},$$ 
where $\eta_i(\by)=\frac{1}{2}\,\frac{f_{\ba_{v_i}}(\by)}{f_{\bzero}(\by)}$.

\vspace{0.22cm}
We start by the following useful lemmas.

%\begin{fact}\label{fact:prob-of-gd}
%	For any $i\in [n],~ v_i\in [J]$, we have $\pr{\tby_i\leftarrow \cRS_i}{\tby_i\in\gd_i}$
%\end{fact} 

\begin{lem}\label{lem:effec-sample-size}
	Suppose $n\geq 120$. With probability $\geq 1-e^{-n/34}$, the number of active users $\hn$ in Step~\ref{step:aggreg} of Algorithm~\ref{Alg:samp-ball} satisfies $\hn > n/4$. 
\end{lem}

\begin{proof}
	Given Algorithm~\ref{Alg:locrnd}, for any user $i\in[n]$, observe that 
	\begin{align*}
		\pr{}{\cRS(v_i)=\bot}=\pr{}{B_i=0}&\leq \pr{\tby_i \sim \cN\left(\mathbf{0},~\sigma^2\right)}{\tby_i\notin\gd_i}+\pr{}{B_i=0~\vert~\tby_i\in\gd_i}\\
		&\leq \delta + (1-\frac{e^{-\eps/4}}{2})\leq\frac{2}{n^2}\,+ \,5/8.
	\end{align*} 
	where the last inequality follows from the fact that $\eps\leq 1$. Thus, we have $\pr{}{B_i=1}\geq 3/8 - 2/n^2.$ Note that $\hn = \sum_{i=1}^n B_i$. Since $n\geq 120$, then by Chernoff's bound, we have
	$$\pr{}{\hn< n/4}< e^{-n/34}.$$
\end{proof}

\begin{lem}\label{lem:ident-distrib}
For any user $i\in [n]$, any input $v_i \in [J]$, and any measurable set $\cO\subseteq \R^d$, we have
	$$\pr{\tby_i\leftarrow\cRS_i(v_i)}{\tby_i\in\cO ~\vert~ B_i=1}=\pr{\tby_i\leftarrow\cRG_i(v_i)}{\tby_i\in\cO ~\vert~ \tby_i\in \gd_i}$$
\end{lem}

\begin{proof}
	Let $\tby_i$ be the Gaussian r.v. generated in Step~\ref{step:delta-set} of Algorithm~\ref{Alg:locrnd}. Note $(\tby_i, B_i)$ has mixed probability distribution. For every realization $\by$ of $\tby_i$ and every $b\in\{0, 1\}$, the joint (mixed) density function of $(\tby_i, B_i)$ can be expressed as $h_{\tby_i\,|\,B_i}(\by|b)\,\pr{}{B_i=b}=p_{B_i \,|\,\tby_i}(b \,|\by)f_{\bzero}(\by)$, where $h_{\tby_i \,|\,B_i}$ is the conditional density of $\tby_i$ given $B_i$ and $p_{B_i\,|\,\tby_i}$ is the conditional density function of $B_i$ given $\tby_i$. Note that we have
	\begin{align}
		p_{B_i\,|\,\tby_i}(1~|~\by)&=\left\{\begin{matrix}
		\frac{1}{2}\,\frac{f_{\ba_{v_i}}(\by)}{f_{\bzero}(\by)} & \by\in\gd_i\\
		 & \\
		0 & \by\notin\gd_i
		\end{matrix}\right.\label{cond-density-B_i}
	\end{align}
	
	Now, observe that for any measurable set $\cO\in\R^d,$
	\begin{align}
		\pr{\tby_i\leftarrow\cRS_i(v_i)}{\tby_i\in\cO ~\vert~ B_i=1}&=\frac{\pr{\tby_i\leftarrow\cRS_i(v_i)}{\tby_i\in\cO\, ,~ B_i=1}}{\pr{}{B_i=1}}\nonumber\\
		&=\frac{\int_{\by\in\cO}p_{B_i\,|\,\tby_i}(1~|~\by)\, f_{\bzero}(\by)\, d\by}{\pr{}{B_i=1}}\nonumber\\
		&=\frac{\frac{1}{2}\int_{\by\in\cO\cap\gd_i}f_{\ba_{v_i}}(\by)\, d\by}{\pr{}{B_i=1}}\label{samp_to_gauss}\\
		&=\frac{1}{2}~ \cdot ~\frac{\pr{\tby_i\leftarrow\cRG_i(v_i)}{\tby_i\in\cO \, ,~ \tby_i\in \gd_i}}{\pr{}{B_i=1}}\label{last-step}
	\end{align} 
where (\ref{samp_to_gauss}) follows from (\ref{cond-density-B_i}), and (\ref{last-step}) follows from observing that the distribution of $\cRG(v_i)$ is $\cN\left(\ba_{v_i},~\sigma^2\,\iden_d\right)$ (whose density is denoted as $f_{\ba_{v_i}}$ as defined early in this section). 

Next, we consider $\pr{}{B_i=1}$. Note that 
\begin{align*}
	\pr{}{B_i=1}&=\pr{\tby_i\leftarrow\cRS_i(v_i)}{B_i=1, \tby_i\in\gd_i}\\
	&=\int_{\by\in\gd_i}p_{B_i\,|\,\tby_i}(1~|~\by)\, f_{\bzero}(\by)\, d\by\\
	&=\frac{1}{2} ~ \cdot ~\pr{\tby_i\leftarrow\cRG_i(v_i)}{\tby_i\in\gd_i}
\end{align*}
Plugging this in (\ref{last-step}), then (\ref{last-step}) reduces to $\pr{\tby_i\leftarrow\cRG_i(v_i)}{\tby_i\in\cO ~\vert~ \tby_i\in \gd_i}$, which proves the lemma.
\end{proof}

\paragraph{Proof of Theorem~\ref{thm:accuracy_samp}:} Putting Lemmas~\ref{lem:effec-sample-size} and \ref{lem:ident-distrib} together with Theorem~\ref{thm:accuracy_gauss} leads us easily to the stated result. Fix any queries matrix $\bA\in\cC_2(r)$. First, note that $\{(\tby_i, B_i):~ i\in [n]\}$ are independent (independence across users). Hence, for any fixed subset $\tilde{\cS}\subseteq [n]$, any sequence of users' inputs $\{v_i: ~i\in\tilde{\cS}\}$, and any sequence of measurable sets $\{\cO_i\subseteq \R^d: ~i\in \tilde{\cS}\},$ we have 
\begin{align}
	\pr{\left\{\tby_i\leftarrow \cRS_i(v_i):~ i\in \tilde{\cS}\right\}}{\tby_i\in\cO_i~ ~\forall i \in \tilde{\cS} ~\vert~ B_i=1 ~~ \forall i \in \tilde{\cS}} &=\prod_{i\in\tilde{\cS}} \pr{\tby_i\leftarrow \cRS_i(v_i)}{\tby_i\in\cO_i ~ \vert~ B_i=1}\nonumber\\
	&=\prod_{i\in\tilde{\cS}}\pr{\tby_i\leftarrow\cRG_i(v_i)}{\tby_i\in\cO_i ~\vert~ \tby_i\in \gd_i}\nonumber\\
	&=\pr{\left\{\tby_i\leftarrow \cRG_i(v_i):~ i\in \tilde{\cS}\right\}}{\tby_i\in\cO_i~ ~\forall i \in \tilde{\cS} ~\vert~ \tby_i\in \gd_i ~ ~ \forall i \in \tilde{\cS}}\nonumber\\
	&\leq \frac{\pr{\left\{\tby_i\leftarrow \cRG_i(v_i):~ i\in \tilde{\cS}\right\}}{\tby_i\in\cO_i~ ~\forall i \in \tilde{\cS}}}{1-\lvert\tilde{\cS}\rvert\cdot\delta}\nonumber\\
	&\leq \left(\frac{n}{n-2}\right)\, \pr{\left\{\tby_i\leftarrow \cRG_i(v_i):~ i\in \tilde{\cS}\right\}}{\tby_i\in\cO_i~ ~\forall i \in \tilde{\cS}}\nonumber\\
	&\leq 1.02 \, \cdot \pr{\left\{\tby_i\leftarrow \cRG_i(v_i):~ i\in \tilde{\cS}\right\}}{\tby_i\in\cO_i~ ~\forall i \in \tilde{\cS}}\label{equiv}
\end{align}
where the second equality follows from Lemma~\ref{lem:ident-distrib}, and the last two inequalities follow from the fact that $\delta = 2/n^2, ~ \lvert\tilde{\cS}\rvert \leq n,$ and the fact that $n\geq 120$.

Let $\cS \subseteq [n]$ denote the subset of active users; that is $\cS=\{i \in [n]: B_i=1\}$. Fix any realization $\tilde{\cS}$ of $\cS$. Let $\tilde{D}=\left\{v_i: ~i\in\tilde{\cS}\right\}$; that is, $\tilde{D}$ is the subset of the data set $D$ of users' inputs congruent with $\tilde{\cS}$. Now, \emph{conditioned on} the event in Lemma~\ref{lem:effec-sample-size} (i.e., conditioned on $\hn\geq n/4$) and \emph{conditioned on} any fixed realization $\tilde{\cS}$ of $\cS$, then by (\ref{equiv}), we have 
\begin{align}
	\ex{\left\{\tby_i\leftarrow \cRS_i(v_i):~ i\in \tilde{\cS}\right\}, ~\tilde{D}}{\norm{\hby-\bA\bp}_2}&\leq 1.02\,\cdot \ex{\left\{\tby_i\leftarrow \cRG_i(v_i):~ i\in \tilde{\cS}\right\}, ~\tilde{D}}{\norm{\hby-\bA\bp}_2} \label{cond_expec}
\end{align}
Note that by Theorem~\ref{thm:accuracy_gauss}, the expectation on the right-hand side is bounded as 

\begin{align}
	\ex{\left\{\tby_i\leftarrow \cRG_i(v_i):~ i\in \tilde{\cS}\right\}, ~\tilde{D}}{\norm{\hby-\bA\bp}_2}&\leq r\cdot \min\left(\left(\frac{32\,\log(J)\log(n^2)}{(n/4)\,\epsilon^2}\right)^{1/4},~ \sqrt{\frac{2\,d\log(n^2)}{(n/4)\,\epsilon^2}}\right)\label{bound2}
\end{align}

By Lemma~\ref{lem:effec-sample-size}, the event $\hn\geq n/4$ occurs with probability at least $1-e^{-n/34}$. Thus, if we remove conditioning on such event from the expectation on the left-hand side of (\ref{cond_expec}), the unconditional version of such expectation can only increase by an additive term of at most $r\,e^{-n/34}$ (since the $L_2$ error cannot exceed $r$, and the probability that $\hn< n/4$ is at most $e^{-n/34}$). This term is  dominated by (\ref{bound2}).

Putting these together, we finally arrive at
$$\err_{\prot_{\samp}, ~L_2}(\cC_2(r), n)\leq ~r\cdot \min\left(\left(\frac{280\,\log(J)\log(n)}{n\epsilon^2}\right)^{1/4},~ \sqrt{\frac{10\,d\log(n)}{n\epsilon^2}}\right).$$

\fi

\subsection{On Tightness of the Bound}\label{subsec:lower_bd}

The above bound (Theorem~\ref{thm:accuracy_samp}) is tight up to a factor \ifnum\final=1 $O\left(\left(\log(J)\log(n)\right)^{1/4}\right).$\else $\approx \left(\log(J)\log(n)\right)^{1/4}$\fi. In particular, one can show a lower bound of \ifnum\final=1$\Omega\left(\min\left(\frac{1}{n^{1/4}\sqrt{\epsilon}},~ \sqrt{\frac{d}{n\epsilon^2}}\right)\right)$ \else $\min\left(\frac{1}{n^{1/4}\sqrt{\epsilon}},~ \sqrt{\frac{d}{n\epsilon^2}}\right)$ \fi on the $L_2$ error. We note that it would suffice to show a tight lower bound on the minimax $L_2$ error in estimating the mean of a $d$-dimensional random variable with a finite support in $\B_2^d(r)$. Such lower bound follows from the lower bound in \cite[Proposition 3]{DJW13}. \ifnum\final=1 We note that the packing constructed in the proof of \cite[Proposition 3]{DJW13} is for a $d$-dimensional random variable with finite support. Hence, this lower bound is applicable to our case where the data universe is of finite size $J$. \fi Tightening the remaining gap between the upper and lower bounds is left as an open problem. We conjecture that the $\log^{1/4}(J)$ factor in the upper bound is necessary.

%% file: distrib.tex
\section{$(\eps, 0)$-LDP Distribution Estimation}\label{sec:distrib}

In this section, we revisit the problem of LDP distribution estimation under $L_2$ error criterion. First, we note that this problem is a special case of the linear queries problem, where the queries matrix $\bA=\iden_J$, i.e., the identity matrix of size $J$. Therefore, our results in Section~\ref{sec:L2_ball_queries} immediately give an upper bound on the $L_2$ error in this case. \ifnum\final=1 Namely, our results imply the existence of pure $\eps$ LDP protocol for distribution estimation whose $L_2$ error is bounded by $\min\left(\left(\frac{280\,\log(J)\log(n)}{n\epsilon^2}\right)^{1/4},~ \sqrt{\frac{10\,J\log(n)}{n\epsilon^2}}\right)$. \fi However, in our protocol $\prot_{\samp}$, both communication complexity and running time per user would be $\Omega(J)$ in this case, which is prohibitive when $J$ is large since the users are usually computationally limited (compared to the server). Our goal in this section is to have a construction with similar error guarantees but with better communication and running time at each user. 

In the low-dimensional setting ($n\gtrsim J/\eps^2$),  \cite{acharya2018communication} give a nice construction (the Hadamard-Response protocol) whose $L_2$ error is $O(\sqrt{\frac{J}{\eps^2 n}})$. In this protocol, both communication and running time at each user are $O(\log(J))$. Also, the running time at the server is $\tilde{O}(n+J)$ (which is significantly better than the naive $O(n\, J)$ running time). We show that this protocol can be extended to the \emph{high-dimensional} setting. In particular, we show that, when $n\lesssim \frac{J^{2}}{\eps^2 \log(J)}$, projecting the output of the Hadamard-Response protocol onto the probability simplex yields $L_2$ error $\lesssim \left(\frac{\log(J)}{\eps^2 n}\right)^{1/4}$, which is tight up to a factor of $\left(\log(J)\right)^{1/4}$ given the lower bounds in \cite{DJW13, ye2018optimal}. This improves the bound of \cite{acharya2018communication} for all $n\lesssim \frac{J^{2}}{\eps^2 \log(J)}$. Moreover, to the best of our knowledge, existing results do not imply $L_2$ error bound better than the trivial $O(1)$ error in the regime where $n\lesssim \frac{J}{\eps^2}$. %Note our bound also improves over the bound of \cite{acharya2018communication} 

%In particular, we show a construction that has $O(\log(J))$ communication and running time per user, and has $\tilde{O}(n+J)$ running time at the server. In particular, we show that projecting the output estimate of the Hadamard-Response protocol of \cite{acharya2018communication} onto the probability simplex yields non-trivial $L_2$ error in the high-dimensional setting ($n\ll J$) similar to (and even slightly better than) what is implied by our bounds for the general case. 
\ifnum\final=1 The idea of projecting the estimated distribution onto the probability simplex was also proposed in \cite{kairouz2016discrete} and was empirically shown to yield improvements in accuracy, however, no formal analysis was provided for the error resulting from this technique.

%\paragraph{Comparison to the error bounds in \cite{acharya2018communication} and \cite{ye2018optimal}:}  In \cite{acharya2018communication}, Acharya et al. gives a nice construction for a pure LDP protocol for estimating a distribution over a domain of size $J$ with small communication and running time per user. The $L_2$ error of their protocol is $O(\sqrt{\frac{d}{\eps^2 n}})$ (when $\eps=O(1)$). 
We want to point out that the $L_2$ error of $\cite{acharya2018communication}$ is optimal only when $n\geq \Omega(J^2/\eps^2)$. Although this condition is not explicitly mentioned in \cite{acharya2018communication}, however, as stated in the same paper, their claim of optimality follows from the lower bound in \cite{ye2018optimal}; specifically, \cite[Theorem IV]{ye2018optimal}. From this theorem, it is clear that the lower bound is only valid when $n\geq \text{const.}\,\frac{J^2}{\eps^2}$. Hence, our bound on the $L_2$ error does not contradict the results of these previous works. 
\fi

 %but as implied by our results, it is not optimal in the high-dimensional setting where $n \leq \tilde{O}(J^2/\eps^2)$. In \cite{acharya2018communication}, although the condition on $n$ for the optimality of their protocol is not explicitly mentioned, their claim of optimality (as they state in the same paper) follows from the lower bound in \cite{ye2018optimal}; specifically, \cite[Theorem IV]{ye2018optimal}. From this theorem, it is clear that the lower bound is only valid when $n\geq \text{const.}\,\frac{J^2}{\eps^2}$. Hence, our bound does not contradict with the results of these previous works. %In fact, our results eliminate the pessimism suggested by the previous works regarding the high-dimensional setting ($J\gg n$), showing that we can have \emph{nearly} dimension independent $L_2$ error in this setting. 
\ifnum\final=1
\paragraph{Outline of Hadamard-Response Protocol of \cite{acharya2018communication}:} We will refer to this protocol as $\prot_{\had}$. We will use such a protocol as a black-box, so, we will not give a detailed description for it. The details can be found in \cite[Section~4]{acharya2018communication}. Let $\tilde{J}=2^{\lceil\log_2(J+1)\rceil}$. Note that $J+1\leq \tilde{J}\leq 2J+1$  Let $H_{\tilde{J}}$ denote the Hadamard matrix of size $\tilde{J}$. As before, the data set $D=\{v_i\in [J]:~i\in [n]\}$ of users' inputs is assumed to be drawn i.i.d. from unknown distribution $\bp=\left(p(1), \ldots, p(J)\right)\in\simp(J)$.

\noindent \textbf{User procedure:} In $\prot_{\had}$, each user $i\in [n]$ encode his input $v_i$ as follows: first, select the $(v_i + 1)$-th row of $H_{\tilde{J}}$, then, encode $v_i$ as the subset $C_{v_i}\subset [\tilde{J}]$ of indices of that row that are incident with $+1$. Given $C_{v_i}$, user $i$ invokes a generalized version of the basic randomized response technique to output a randomized index  $z_i\in [\tilde{J}]$ as its $\eps$-LDP report  (See \cite[Section~3]{acharya2018communication}). Hence, the communication requirement per user is $\leq \log_2(2J+1)$ bits. Moreover, by the properties of the Hadamard matrix and the generalized randomized response, all the operations at any user can be executed in time $O(\log(J))$. 

\noindent \textbf{Server procedure:} For every element $v\in[J]$ in the domain, the server generates an estimate $\bar{p}(v)$ for the true probability mass $p(v)$ as follows: 
\begin{align}
	\bar{p}(v)=\left(\frac{e^{\eps}+1}{e^{\eps}-1}\right)\cdot \sum_{w\in[\tilde{J}]}H_{\tilde{J}}(v+1, w)\cdot q(w),\label{server_pmf_est}
\end{align}
where $H_{\tilde{J}}(v+1, w)$ denotes the entry of $H_{\tilde{J}}$ at the $(v+1)$-th row and the $w$-th column, and $q(w)=\frac{1}{n}\sum_{i=1}^n \bone(v_i=w)$ is the fraction of users whose reports are equal to $w$ (See \cite[Section~4]{acharya2018communication}). Finally, the server outputs the vector of estimates: $\bar{\bp}=\left(\bar{p}(1), \dots, \bar{p}(J)\right)$. As shown in \cite{acharya2018communication}, the total operations can be done in $\tilde{O}(n+J)$. 

As noted in the same reference, equation (\ref{server_pmf_est}) reduces to 
\begin{align}
\bar{p}(v)=2\,\left(\frac{e^{\eps}+1}{e^{\eps}-1}\right)\cdot \left(\widehat{q(C_v)}-\frac{1}{2}\right),\label{eq_pmf_est}
\end{align}
where $\widehat{q(C_v)}=\frac{1}{n}\sum_{i=1}^{n}\bone\left(z_i\in C_v\right)$, which, by the properties of the local randomization, is the average of $n$ Bernoulli r.v.s. Moreover, as shown in \cite[Section~3]{acharya2018communication}, $\ex{\prot_{\had}}{\bar{p}(v)}=p(v)~~ \forall v\in [J]$. Putting these observations together with Chernoff's inequality leads to the following fact.

\else
\paragraph{Hadamard-Response Protocol of \cite{acharya2018communication}:} In the attached full version, we give a brief description of this protocol for completeness. Due to limited space, we only give here the relevant facts about this protocol, which we denote as $\prot_{\had}$: (i) $\prot_{\had}$ is $\eps$-LDP. (ii) $\prot_{\had}$ requires $O(\log(J))$ bits of communication per user, and the running time at each user is also $O(\log(J))$. (iii) The running time at the server is $\tilde{O}(n+J)$. (iv) When $n\gtrsim J/\eps^2$ (low-dimensional or large sample setting), then for any true distribution $\bp\in\simp(J),$ $\prot_{\had}$ outputs an estimate $\bar{\bp}$ for $\bp$ satisfying: $\ex{\prot_{\had}, D\sim\bp^n}{\norm{\bar{\bp}-\bp}_2}=O\left(\sqrt{\frac{J}{\eps^2 n}}\right)$. Following the steps of this protocol, one can show the following fact (see the attached full version for details).   
\fi

\begin{fact}\label{fact:sub-gauss}
	Let $\sigma^2=4\left(\frac{e^{\eps}+1}{e^{\eps}-1}\right)^2$. Each of the $J$ components of ~$\bar{\bp}-\bp$ ~ is $\frac{\sigma^2}{n}$-subGaussian random variable.
\end{fact}
\ifnum\submission=1 (See the Preliminaries section of the full version for a formal definition of subGaussian random variables).\fi

We now \ifnum\final=1 give a construction \else describe \fi $\prot_{\phad}$ (Projected Hadamard-Response) \ifnum\final=1 where the output estimate of $\prot_{\had}$ is projected onto the probability simplex whenever $n\lesssim \frac{J^2}{\eps^2\log(J)}$. \fi
%In this case, the server applies a projection step similar to the one we used in Section~\ref{sec:L2_ball_queries}. We show that this simple extra step leads to a non-trivial error bound in the high-dimensional setting, which depends only sub-logarithmically on $J$. This shows the possibility of distribution estimation under $L_2$ error criterion in that setting. Moreover,  our construction 
\ifnum\final=1 Clearly, \fi$\prot_{\phad}$ has the same computational advantages of $\prot_{\had}$.%, particularly, the $O(\log(J))$ communication and running time at every user.

\begin{algorithm}[htb]
	\caption{$\prot_{\phad}$: $\eps$-LDP protocol for distribution estimation}
	\begin{algorithmic}[1]
		\REQUIRE Data set of users' inputs $D=\{v_i\in [J]: i\in[n]\}$, privacy parameter $\eps$.
		\STATE  $\bar{\bp}\leftarrow\prot_{\had}(D,~\eps)$
		%\IF {$n < \left(\frac{e^{\eps}+1}{e^{\eps}-1}\right)^2\,\frac{J^2}{16\log(J)}$:}\label{step:high-dim-cond}
		{\STATE $\hat{\bp}=\arg\min_{\bw\in \simp(J)}\|\bw-\bar{\bp}\|_2^2$.}
		%\ELSE {\STATE $\hat{\bp}=\bar{\bp}$}
		%\ENDIF
		\RETURN $\hat{\bp}$.
	\end{algorithmic}
\label{Alg:projhad}
\end{algorithm}

\ifnum\final=1
First, note that since differential privacy is resilient to post-processing, $\eps$-LDP of $\prot_{\phad}$ immediately follows from $\eps$-LDP of $\prot_{\had}$ (shown in \cite{acharya2018communication}). 
\else
%Note that 
Note that $\prot_{\phad}$ is $\eps$-LDP since $\prot_{\had}$ is $\eps$-LDP . 

\fi

\ifnum\final=1

\begin{thm}[Accuracy of Algorithm~\ref{Alg:projhad}]\label{thm:accuracy_projhad}
	Let $c_{\eps}\triangleq \frac{e^{\eps}+1}{e^{\eps}-1}$. (Note that $c_{\eps}=O(1/\eps)$ when $\eps=O(1)$). Protocol $\prot_{\phad}$ satisfies the following accuracy guarantee:
	$$\err_{\prot_{\phad}, ~L_2}( n)\triangleq  \sup_{\bp\in \simp(J)}~\ex{\prot_{\phad}, ~D\sim\bp^n}{\|\prot_{\phad}(D)-\bp\|_2}\leq ~ \min\left(\left(\frac{256\,c^2_{\eps}\,\log(J)}{n}\right)^{1/4},~ \sqrt{\frac{4\,c^2_{\eps}\,J}{n}}\right).$$
\end{thm}

\else
\begin{thm}\label{thm:accuracy_projhad}
	Let $c_{\eps}\triangleq \frac{e^{\eps}+1}{e^{\eps}-1}$. %(Note that $c_{\eps}=O(1/\eps)$ when $\eps=O(1)$). 
	$\prot_{\phad}$ has $L_2$ error $$\err_{\prot_{\phad}, ~L_2}( n)\triangleq\hspace{-0.07cm}\sup_{\bp\in \simp(J)}\ex{\prot_{\phad}, ~D\sim\bp^n}{\|\prot_{\phad}(D)-\bp\|_2}$$ that is upper bounded by
	
	$$ \min\left(\left(\frac{256\,c^2_{\eps}\,\log(J)}{n}\right)^{1/4},~ \sqrt{\frac{4\,c^2_{\eps}\,J}{n}}\right).$$
\end{thm}

\fi

\ifnum\final=1

\begin{proof}
	Fix any $\bp\in\simp(J)$ as the true distribution. First, consider the case where $n \geq \left(\frac{e^{\eps}+1}{e^{\eps}-1}\right)^2\,\frac{J^2}{16\log(J)}$. Note that in this case the bound follows from \cite{acharya2018communication} (since the projection step cannot increase the $L_2$ error). 
	
	Next, we consider the case where $n < \left(\frac{e^{\eps}+1}{e^{\eps}-1}\right)^2\,\frac{J^2}{16\log(J)}$. Note that the symmetric version of the polytope $\simp(J)$ is the $L_1$ Ball $\B_1^J$. Let $\bp^*=\arg\min_{\bw\in\B_1^J}\norm{\bw-\bar{\bp}}_2^2$. Corollary~\ref{proj_cor} tells us that 
	$$\norm{\bp^*-\bp}_2^2\leq 4\max\limits_{j\in [J]}\lvert\langle \bar{\bp}-\bp,~\be_j\rangle\rvert,$$
	where $\be_j\in\R^J$ denotes the vector with $1$ in the $j$-th coordinate and zeros elsewhere. Now, as defined in $\prot_{\phad}$, let $\hat{\bp}=\arg\min_{\bw\in \simp(J)}\|\bw-\bar{\bp}\|_2^2$. Since $\bp\in\simp(J)$, then for the special case where the symmetric polytope is $\B_1^J$, we always have  $\norm{\hat{\bp}-\bp}_2\leq \norm{\bp^*-\bp}_2$. This is because $\bp^*-\bp$ in this case can be written as the sum of two orthogonal components : $\left(\bp^*-\hat{\bp}\right) ~+~\left(\hat{\bp}-\bp\right).$ Hence, Corollary~\ref{proj_cor} implies that 
	$$\|\hat{\bp}-\bp\|_2^2\leq 4\max\limits_{j\in [J]}\lvert\langle \bar{\bp}-\bp,~\be_j\rangle\rvert.$$  
	By Fact~\ref{fact:sub-gauss}, for every $j\in [J]$, $\langle \bar{\bp}-\bp,~\be_j\rangle$ is $\frac{\sigma^2}{n}$-subGaussian where $\sigma^2=4\left(\frac{e^{\eps}+1}{e^{\eps}-1}\right)^2$. Now, by using the standard bounds on the  maximum of subGaussian r.v.s (see \cite[Theorem~1.16]{mit}), we have 
	$$\ex{\prot_{\phad}, ~D\sim\bp^n}{\norm{\hat{\bp}-\bp}_2^2}\leq 4\frac{\sigma}{\sqrt{n}}\,\sqrt{2\,\log(2J)}\leq \sqrt{\frac{256\, c^2_{\eps}\,\log(J)}{n}}.$$
	Thus, we have $\err_{\prot_{\phad}, ~L_2}( n)\leq \left(\frac{256\,c^2_{\eps}\,\log(J)}{n}\right)^{1/4}.$

\end{proof}

\else

\begin{proof}
	Fix any $\bp\in\simp(J)$ as the true distribution. First, consider the case where $n \geq \left(\frac{e^{\eps}+1}{e^{\eps}-1}\right)^2\,\frac{J^2}{16\log(J)}$. Note that in this case the bound follows from \cite{acharya2018communication} (since the projection step cannot increase the $L_2$ error). 
	
	Next, we consider the case where $n < \left(\frac{e^{\eps}+1}{e^{\eps}-1}\right)^2\,\frac{J^2}{16\log(J)}$. Note that the symmetric version of the polytope $\simp(J)$ is the $L_1$ Ball $\B_1^J$. Let $\bp^*=\arg\min_{\bw\in\B_1^J}\norm{\bw-\bar{\bp}}_2^2$. Fact~\ref{proj_cor} tells us that 
	$\norm{\bp^*-\bp}_2^2\leq 4\max\limits_{j\in [J]}\lvert\langle \bar{\bp}-\bp,~\be_j\rangle\rvert,$
	where $\be_j\in\R^J$ denotes the vector with $1$ in the $j$-th coordinate and zeros elsewhere. Now, as defined in $\prot_{\phad}$, let $\hat{\bp}=\arg\min_{\bw\in \simp(J)}\|\bw-\bar{\bp}\|_2^2$. Since $\bp\in\simp(J)$, then for the special case where the symmetric polytope is $\B_1^J$, we always have  $\norm{\hat{\bp}-\bp}_2\leq \norm{\bp^*-\bp}_2$. This is because $\bp^*-\bp$ in this case can be written as the sum of two orthogonal components : $\left(\bp^*-\hat{\bp}\right) ~+~\left(\hat{\bp}-\bp\right).$ Hence, Fact~\ref{proj_cor} implies that $\|\hat{\bp}-\bp\|_2^2\leq 4\max\limits_{j\in [J]}\lvert\langle \bar{\bp}-\bp,~\be_j\rangle\rvert.$  
	By Fact~\ref{fact:sub-gauss}, for every $j\in [J]$, $\langle \bar{\bp}-\bp,~\be_j\rangle$ is $\frac{\sigma^2}{n}$-subGaussian. Thus, by using the standard bounds on the  maximum of subGaussian r.v.s (see \cite[Theorem~1.16]{mit}), we have $\ex{\prot_{\phad}, ~D\sim\bp^n}{\norm{\hat{\bp}-\bp}_2^2}\leq \sqrt{\frac{256\, c^2_{\eps}\,\log(J)}{n}}.$
%	Thus, we have $\err_{\prot_{\phad}, ~L_2}( n)\leq \left(\frac{256\,c^2_{\eps}\,\log(J)}{n}\right)^{1/4}.$
	
\end{proof}

\fi

%% file: adaptive.tex
\section{$\eps$-LDP \ifnum\final=1 Protocol \fi for Adaptive Linear Queries }\label{sec:adap-Linf_ball_queries}

\ifnum\final=1
In this section, we consider the problem of estimating any sequence of $d$ adaptively chosen linear queries $\bq_1, \ldots, \bq_d$ from $\cQ_{\infty}(r)$ over some unknown distribution $\bp\in\simp(J)$. As defined in Section~\ref{sec:acc-def-adaptive}, for any fixed $r>0$, each query from $\cQ_{\infty}(r)$ is a linear function $\langle \bq,~ \cdot\rangle,$ which is uniquely identified by a vector $\bq\in \R^J$ where $\|\bq\|_{\infty}\leq r$.% (hence, there is a one-to-one correspondence between $\cQ_{\infty}(r)$ and the $L_{\infty}$ ball $\B_{\infty}^J$). 

\else
We now consider the problem of estimating a sequence of $d$ adaptively chosen linear queries $\bq_1, \ldots, \bq_d$ from $\cQ_{\infty}(r)$ over unknown distribution $\bp\in\simp(J)$.
\fi
%In this setting, the queries are submitted and answered over $d$ rounds: one query in each round. The server can adaptively choose the next query based on all the history it sees, i.e., based on all the previous queries and answers. The goal is to design an $\eps$-LDP protocol where in each round $k\in[d]$, the users send privatized (randomized) signals that depend on their inputs $\{v_i:~ i\in [n]\}\sim\bp^n$ as well as the current query $\bq_k$, which enable the server to produce an accurate estimate of $\langle \bq_k,~ \bp \rangle$. 
We measure accuracy in terms of the $L_{\infty}$ estimation error in the $d$ queries as defined in (\ref{linf-err}) in Section~\ref{sec:acc-def-adaptive}. 
%That is, 
%$$\err_{\prot, L_{\infty}}(\cQ_{\infty}(r), d, n)=\sup\limits_{\bp\in\simp(J)}~~ \sup\limits_{\substack{\text{adaptive strategy}\\ \text{choosing }~\bq_1, \ldots, \bq_d}}~ \ex{\prot, ~D\sim\bp^n}{\max\limits_{k\in [d]}~\lvert~ \prot^{(k)}(D) - \langle \bq_k, ~\bp\rangle ~\rvert~},$$ where $\prot^{(k)}(D)$ denotes the estimate generated by the protocol in the $k$-th round.

We give a construction of $\eps$-LDP protocol that yields the optimal $L_{\infty}$ error. The optimality follows from the fact that our upper bound matches a lower bound on the same error in the \emph{weaker non-adaptive} setting, which follows from the lower bound in \cite[Proposition~4]{DJW13}. Moreover, in our protocol each user sends only $O(\log(r))$ bits to the server\ifnum\final=1\footnote{Assuming fixed-precision representation of real numbers in $[-1,~1]$}\fi. In our protocol, the set of users are \emph{randomly} partitioned into $d$ disjoint subsets before the protocol starts, and each subset is used to answer one query. Assignment of the subsets to the queries is \emph{fixed before the protocol starts}. Roughly speaking, this partitioning technique can be viewed as sample splitting. This avoids the trap of overfitting a query to the data samples it is evaluated on. %which is the main challenge in adaptive data analysis. 
%\ifnum\final=1 
In the centralized model, sample-splitting is generally sub-optimal. Our result shows that for adaptive linear queries in the local model, this technique is optimal. % \fi 

%\subsection{An $(\epsilon, 0)$ Local Protocol for Adaptive Linear Queries}\label{subsec:adap-samp}

The description of the protocol is given in Algorithm~\ref{Alg:adap-samp}. 

\begin{algorithm}[htb]
	\caption{$\prot_{\adapsamp}$: $\eps$-LDP protocol for adaptive linear queries from $\cQ_{\infty}(r)$}
	\begin{algorithmic}[1]
		\REQUIRE Data set \ifnum\final=1 of users' inputs \fi $D=\{v_i\in [J]: i\in[n]\}$, privacy parameter $\eps$,~ \ifnum\final=1 sequence of \fi $d$ adaptive \ifnum\final=1 linear \fi queries \ifnum\final=1 $\bq_1, \ldots, \bq_d\in \cQ_{\infty}(r)$ \else from $\cQ_{\infty}(r)$ \fi.
{\STATE Each user $i\in [n]$ gets independently assigned (by itslef or via the server) a random uniform index $j_i\leftarrow [d]$.\label{step:uindex}} \ifnum\final=1\COMMENT{This creates a random partition of the users, which is independent of the users' data, and is fixed before the protocol starts.}\fi
\FOR{ $k=1, \ldots, d$}
\FOR{ all users $i$ \textbf{such that} $j_i=k$}
	   {\STATE User $i$ receives query $\bq_k$ responds with $\tilde{y}_{k,i}$ generated as follows: 
                            \begin{align}
			        &\tilde{y}_{k,i}\hspace{-.1cm}=\left\{\begin{array}{cc}
                                                        c_{\eps} r   & \text{ w.p. } \frac{1}{2}\left(1+\frac{\bq_k(v_i)}{c_{\eps}r}\right) \\
					           -c_{\eps} r   & \text{ w.p. } \frac{1}{2}\left(1-\frac{\bq_k(v_i)}{c_{\eps}r}\right)
					                 \end{array}\right.\nonumber
			    \end{align} 
			where $c_{\eps}=\frac{e^{\eps}+1}{e^{\eps}-1}$. \label{step:adap-local_rand}}
        \ENDFOR
	  {\STATE Server computes an estimate $\bar{y}_k=\frac{1}{\hn_k}\sum_{i: j_i=k}\tilde{y}_{k,i}$ \ifnum\final=1 based on the reports of the active users in round $k:~\{\tilde{y}_{k,i}:~j_i=k\}$, \fi where $\hn_k=\lvert\{i\in [n]:~ j_i=k\}\rvert$ is the number of active users in round $k$.\label{step:adap-aggreg}}
	  	\STATE Server chooses a new query $\bq_{k+1}\in \cQ_{\infty}(r)$ \ifnum\final=1 (possibly based on all observations it received from the users until round $k$).\fi
\ENDFOR
%{\STATE Server computes $\bar{\by}=\sum_{i=1}^n\tilde{\by}_i$ where $\tilde{\by}_i\triangleq \left(0, \ldots, 0, \tilde{y}_{j_i,i}, 0, \ldots, 0\right)\in\R^d$ is the answer vector of user $i$, that is, $\tilde{\by}_i$ is a vector with $y_{j_i, i}$ in the $j_i$-th position and zeros elsewhere.}
%	\IF {$n < 16d^2/\epsilon^2$:} 
%		{\STATE $\hat{\by}=\arg\min_{\bz\in n\bA\B_1^N}\|\bz-\bar{\by}\|_2^2$ where $\bA\in \R^{d\times N}$ is the matrix whose rows are the queries vectors $\bq_1,\ldots, \bq_d$, and $\B_1^N$ is the $L_1$ ball in $\R^N$.}
%	\ELSE {\STATE $\hat{\by}=\bar{\by}$}
%	\ENDIF
	\RETURN Estimated vector $\bar{\by}=(\bar{y}_1, \ldots, \bar{y}_d)$.
	\end{algorithmic}
	\label{Alg:adap-samp}
\end{algorithm}

\ifnum\final=1

\begin{thm}[Privacy Guarantee]\label{thm:privacy_adap-samp}
Protocol $\prot_{\adapsamp}$ given by Algorithm~\ref{Alg:adap-samp} is $(\epsilon, 0)$-LDP.
\end{thm}
\begin{proof}
Fix any user $i$ and any choice of $j_i =k \in [d]$. Observe that user $i$ responds only to a single query: the $k$-th query. We show that the procedure described by Step~\ref{step:adap-local_rand} is $\eps$-differentially private with respect to any such user's input item. First, note that $c_{\eps}\geq 1$ for all $\eps>0$ and $|\bq_{k}|\leq r$. Hence, it is easy to verify that $\frac{1}{2}\left(1+\frac{\bq_k(v_i)}{c_{\eps}r}\right)$ and $\frac{1}{2}\left(1-\frac{\bq_k(v_i)}{c_{\eps}r}\right)$ in Step~\ref{step:adap-local_rand} are always in $(0,1)$ and they sum to $1$, so they are legitimate probabilities. Observe that the ratio of the probabilities of the responses of user $i$ when its data item is $v_i$ and $v'_i$ is given by 
\begin{align}
\frac{\pr{}{\tilde{y}_{k,i}=c_{\eps}\,r\big\vert ~v_i}}{\pr{}{\tilde{y}_{k,i}=c_{\eps}\,r\big\vert ~v'_i}}&=\frac{c_{\eps}+\frac{\bq_k(v_i)}{r}}{c_{\eps}+\frac{\bq_k(v'_i)}{r}}\leq \frac{c_{\eps}+1}{c_{\eps}-1}=e^{\eps}\nonumber
\end{align}
where the second inequality follows from the fact that $\vert \bq_k(v) \vert\leq r$ for all $k\in [d]$ and all $v\in [J]$. Similarly, 
\begin{align}
\frac{\pr{}{\tilde{y}_{k,i}=-c_{\eps}\,r\big\vert ~v_i}}{\pr{}{\tilde{y}_{k,i}=-c_{\eps}\,r\big\vert ~v'_i}}&=\frac{c_{\eps}-\frac{\bq_k(v_i)}{r}}{c_{\eps}-\frac{\bq_k(v'_i)}{r}}\leq \frac{c_{\eps}+1}{c_{\eps}-1}=e^{\eps}\nonumber
\end{align}

\end{proof}

\else

\begin{thm}\label{thm:privacy_adap-samp}
	Algorithm~\ref{Alg:adap-samp} is $\epsilon$-LDP.
\end{thm}
\begin{proof}
	Fix any user $i$ and any choice of $j_i =k \in [d]$. Observe that user $i$ responds only to a single query: the $k$-th query. First, note that $c_{\eps}\geq 1$ and $|\bq_{k}|\leq r$. Hence,  $\frac{1}{2}\left(1+\frac{\bq_k(v_i)}{c_{\eps}r}\right)$ and $\frac{1}{2}\left(1-\frac{\bq_k(v_i)}{c_{\eps}r}\right)$ in Step~\ref{step:adap-local_rand} are legitimate probabilities. Observe that for any pair $v_i, v'_i \in [J]$,
	\begin{align}
	\frac{\pr{}{\tilde{y}_{k,i}=c_{\eps}\,r\big\vert ~v_i}}{\pr{}{\tilde{y}_{k,i}=c_{\eps}\,r\big\vert ~v'_i}}&=\frac{c_{\eps}+\frac{\bq_k(v_i)}{r}}{c_{\eps}+\frac{\bq_k(v'_i)}{r}}\leq \frac{c_{\eps}+1}{c_{\eps}-1}=e^{\eps}\nonumber
	\end{align}
	%where the second inequality follows from the fact that $\vert \bq_k(v) \vert\leq r$ for all $k\in [d]$ and all $v\in [J]$. 
	We can bound the ratio of the probabilities for the case of $\ty_{k, i}=-c_{\eps}\,r$ in a similar fashion.
%	\begin{align}
%	\frac{\pr{}{\tilde{y}_{k,i}=-c_{\eps}\,r\big\vert ~v_i}}{\pr{}{\tilde{y}_{k,i}=-c_{\eps}\,r\big\vert ~v'_i}}&=\frac{c_{\eps}-\frac{\bq_k(v_i)}{r}}{c_{\eps}-\frac{\bq_k(v'_i)}{r}}\leq \frac{c_{\eps}+1}{c_{\eps}-1}=e^{\eps}\nonumber
%	\end{align}
%	
\end{proof}

\fi

%\begin{lem}\label{lem:adap-unbiased}
%For every user $i\in [n]$, the response vector $\tilde{\by}_i$ defined in Step~\ref{step:adap-aggreg} is unbiased, i.e., 
%$$\ex{}{\tilde{\by}_i}=\left(\bq_1(v_i), \bq_2(v_i), \ldots, \bq_d(v_i)\right)$$
%where the expectation is taken over all randomness in the protocol (i.e., the randomness in $j_i$ and the randomness in Step~\ref{step:adap-local_rand}).
%\end{lem}
%
%\begin{proof}
%Conditioned on a fixed choice of $j_i =j \in [d]$, observe that the expected value of $\tilde{y}_{j, i}$ in Step~\ref{step:adap-local_rand} is $d\bq_j(v_i)$. Since $j_i$ is chosen uniformly from $[d]$, the result follows immediately.
%\end{proof}

\ifnum\final=1
\begin{thm}[Accuracy]\label{thm:accuracy_adap-samp}
Suppose $n\geq 8\,d\,\log(n)$. Then, Protocol $\prot_{\adapsamp}$ given by Algorithm~\ref{Alg:adap-samp} satisfies the following accuracy guarantee for any sequence $\bq_1, \ldots, \bq_d \in\cQ_{\infty}(r)$ of adaptive linear queries
$$\err_{\prot_{\adapsamp},~ L_{\infty}}(\cQ_{\infty}(r), d, n)\leq 4\,r\,~\sqrt{\frac{c^2_{\eps} d\log(d)}{n}},$$
where $\err_{\prot_{\adapsamp},~ L_{\infty}}(\cQ_{\infty}(r), d, n)$ is as defined in (\ref{linf-err}). 
Moreover, this bound is optimal.
\end{thm}
\else
\begin{thm}\label{thm:accuracy_adap-samp}
	Suppose $n\geq 8\,d\,\log(n)$. Then, $\prot_{\adapsamp}$ satisfies the following accuracy guarantee for any sequence $\bq_1, \ldots, \bq_d \in\cQ_{\infty}(r)$ of adaptive linear queries
	$$\err_{\prot_{\adapsamp},~ L_{\infty}}(\cQ_{\infty}(r), d, n)\leq 4\,r\,~\sqrt{\frac{c^2_{\eps} d\log(d)}{n}},$$
	where $\err_{\prot_{\adapsamp},~ L_{\infty}}(\cQ_{\infty}(r), d, n)$ is as defined in (\ref{linf-err}).
	
	Moreover, this bound is optimal.
\end{thm}
\fi

\ifnum\submission=1

The proof of the above theorem is provided in the attached full version. Given the \emph{offline} lower bound in \cite{DJW13}, this result shows that adaptivity does not pose any extra penalty in the \emph{true} $L_{\infty}$ error for linear queries in the local model. In contrast, it is still not clear whether the same statement can be made about linear queries in the \emph{centralized} model. For instance, assuming $\eps = \Theta(1)$ and $n\gtrsim d^{3/2},$ then in the \emph{centralized} model, the best known upper bound on the \emph{true} $L_{\infty}$ estimation error in the \emph{adaptive} setting is $\approx d^{1/4}/\sqrt{n}$~ \cite[Corollary~6.1]{BNS+16} (which combines \cite{DMNS} with the generalization guarantees of differential privacy). Whereas in the offline setting, the \emph{true} $L_{\infty}$ error is upper-bounded by $\approx \sqrt{\frac{\log(d)}{n}}$ (combining \cite{DMNS} with the standard generalization bound for the offline setting). There is also a gap to be tightened in the other regime of $n$ and $d$ as well. This, for example, can be seen by comparing \cite[Corollary~6.3]{BNS+16} with the bound attained by \cite{hardt2010multiplicative} in the offline setting.

\else

Our upper bound is optimal since it matches a lower bound on the $L_{\infty}$ error in the \emph{weaker non-adaptive} version of the same problem, which follows from \cite[Proposition~4]{DJW13} . %This establishes the optimality of our upper bound in the \emph{adaptive} setting. %of this problem since our bound matches exactly the lower bound in \cite[Proposition~4]{DJW13} (note that $c_{\eps}=O(1/\eps)$ when $\eps=O(1)$). 
%Interestingly, unlike what is currently known about this problem in the centralized model of differential privacy, there is no extra penalty in the \emph{true} worst-case error due to adaptivity for linear queries in the local model. That is, adaptive linear queries cannot worsen the true error more than non-adaptive queries in the local model. In the centralized model, the best known upper bound on the true $L_{\infty}$ error for this problem is $\approx \frac{d^{1/4}}{\sqrt{n}}$ \cite[Corollary~6.1]{BNS+16}, whereas in the offline setting, the upper bound on the true error is $\approx \frac{\sqrt{d}}{n}+\sqrt{\frac{\log(d)}{n}}$ \cite{DMNS}, which is clearly smaller than the bound in the adaptive case by a factor of $\approx \min\left(\frac{\sqrt{n}}{d^{1/4}}, ~\frac{d^{1/4}}{\sqrt{\log(d)}}\right)$ for all $d \lesssim n^2$.

This result shows that adaptivity does not pose any extra penalty in the \emph{true} $L_{\infty}$ estimation error for linear queries in the local model. In contrast, it is still not clear whether the same statement can be made about linear queries in the \emph{centralized} model of differential privacy. For instance, assuming $\eps = \Theta(1)$ and $n\gtrsim d^{3/2},$ then in the \emph{centralized} model, the best known upper bound on the \emph{true} $L_{\infty}$ estimation error in the \emph{adaptive} setting is $\approx d^{1/4}/\sqrt{n}$ \cite[Corollary~6.1]{BNS+16} (which combines \cite{DMNS} with the generalization guarantees of differential privacy). Whereas in the offline setting, the \emph{true} $L_{\infty}$ error is upper-bounded by $\approx \sqrt{\frac{\log(d)}{n}}$ (combining \cite{DMNS} with the standard generalization bound for the offline setting). There is also a gap to be tightened in the other regime of $n$ and $d$ as well. This, for example, can be seen by comparing the bound for the adaptive setting \cite[Corollary~6.3]{BNS+16} with the bound attained by \cite{hardt2010multiplicative} in the offline setting.

% In the centralized model, assuming $\eps = \Theta(1)$, the best known upper bound on the \emph{true} $L_{\infty}$ estimation error for this problem in the \emph{adaptive} setting is $\approx \log^{1/6}(J)\left(\frac{\log(d)}{n}\right)^{1/3}$ \cite[Corollary~6.3]{BNS+16} (which combines the bound attained by the private multiplicative weights algorithm \cite{hardt2010multiplicative} with the generalization guarantees of differential privacy). Whereas in the offline setting, the \emph{true} $L_{\infty}$ error is upper-bounded by $\approx \log^{1/4}(J)\sqrt{\frac{\log(d)}{n}}$ (combining \cite{hardt2010multiplicative} with the standard generalization error bound in the offline setting).

%Unlike what is currently known about this problem in the centralized model, our result shows that in the local model there is no extra penalty in the \emph{true} worst-case error due to adaptivity for the case of linear queries. That is, adaptive linear queries cannot worsen the true error more than non-adaptive queries under LDP. In the centralized model, the best known upper bound on the true $L_{\infty}$ error for this problem is $\approx \frac{d^{1/4}}{\sqrt{n}}$ \cite[Corollary~6.1]{BNS+16}, whereas in the offline setting, an upper bound on the true error is $\approx \frac{\sqrt{d}}{n}+\sqrt{\frac{\log(d)}{n}}$ \cite{DMNS}. Clearly, the bound for the offline case is better than the best known bound in the adaptive case by a factor of $\approx \min\left(\frac{\sqrt{n}}{d^{1/4}}, ~\frac{d^{1/4}}{\sqrt{\log(d)}}\right)$ for all $d \lesssim n^2$.

\begin{proof}

Let $\bp$ denote the true distribution over the data domain $[J]$. Let $D=\{v_i:~ i\in[n]\}\sim \bp^n$. 
For every $i\in[n],~ k\in[d],$ define $B_{k, i}=\bone(j_i=k),$ where $j_i\leftarrow [d]$ is the uniform index generated for user $i$ in Step~\ref{step:uindex} of $\prot_{\adapsamp}$. Note that for any fixed $k\in [d]$, $\{B_{k, i}:~i\in [n]\}$ are i.i.d., and $\pr{}{B_{k, i}=1}=1/d$ for every $i\in [n]$. For every $k\in [d],$ define $\cI_k=\left\{i\in [n]:~B_{k, i}=1\right\}$; that is, $\cI_k$ is the set of active users in round $k$. Hence, $\hn_k=\lvert\cI_k\rvert$, where $\hn_k$ is the number of active users (as in Step~\ref{step:adap-aggreg}). Let $D_k\subseteq D$ be defined as $D_k=\{v_i:~i\in\cI_k\}$; that is, $D_k$ is the subset of data set $D$ that contains the inputs of the active users in round $k$. 
%Define a collection of independent random variables $\left\{z_{k, i}:~i\in [n],~ k\in [d]\right\}$, where each $z_{k, i}$ is generated independently by the following procedure:
%\begin{align}
%&z_{k, i}\hspace{-.1cm}=\left\{\begin{array}{cc}
%c_{\eps} \,r   & \text{ w.p. } \frac{1}{2}\left(1+\frac{\bq_k(v_i)}{c_{\eps}r}\right) \\
%-c_{\eps} \,r   & \text{ w.p. } \frac{1}{2}\left(1-\frac{\bq_k(v_i)}{c_{\eps}r}\right)
%\end{array}\right.\nonumber
%\end{align} 
%Note that since for each $i\in [n],~ v_i\sim \bp,$ then for all $k\in [d],~ i\in [n],$ 
For every round $k\in [d],$ as in Step~\ref{step:adap-aggreg}, $\bar{y}_k$ is given by 
\begin{align}
\bar{y}_k&=\frac{1}{\hn_k}\sum_{i\in\cI_k} \ty_{k, i}.\label{eq:bar_y}
\end{align} 

Suppose we \emph{condition on} any fixed realization of the partition of users $\{\cI_k:~ k\in [d]\}$. Conditioned on any such partition, since the users' inputs in $D$ are i.i.d., then $D_1, \ldots, D_k$ are mutually independent. Hence, conditioned on any such partition, for every round $k\in [d],$ \emph{the choice of the query $\bq_k$ is independent of the subsample $D_k$} involved in the computation in round $k$, and thus, for every round $k\in [d],$ we have

\begin{align*}
\ex{}{\ty_{k, i}\big\vert~\left(\cI_{1}, \ldots, \cI_d\right)}&= \ex{}{\ty_{k, i}~\big\vert~\cI_{k}}=\langle \bq_k, \bp\rangle ~~~\forall i \in\cI_k %\label{eq:expec-per-user}
\end{align*}
where the expectation is taken w.r.t. $v_i\sim \bp$, the randomization Step~\ref{step:adap-local_rand} in $\prot_{\adapsamp}$, and any possible randomness in the choice of $\bq_k$. Hence, \emph{conditioned on any fixed partition of the users}, from (\ref{eq:bar_y}) we get
%fixed realization for all $B_{k, i}, i\in[n], ~k\in [d],$ 

%Thus, conditioned on any such realization of the partition, for every round $k$, from (\ref{eq:expec-per-user})-(\ref{eq:bar_y}) we have
\begin{align}
	\ex{\prot_{\adapsamp},~D\sim \bp^n}{\bar{y}_k~\big\vert~\left(\cI_{1}, \ldots, \cI_d\right)}&=\frac{1}{\hn_k}~\sum_{i\in\cI_k}~~\ex{\prot_{\adapsamp},~D_k\sim \bp^{\hn_k}}{\ty_{k, i}\big\vert~\cI_{k}}=~\langle \bq_k, \bp\rangle ~~~~~ \forall k\in [d]\label{cond-expec-bar_y}
\end{align}
Now, we define the set $\gd$ that contains ``good'' realizations for the partition $(\cI_1, \ldots, \cI_d)$:
\begin{align}
	\gd&=\left\{(\cI_1, \ldots, \cI_d):~ \lvert\cI_k\rvert\geq \frac{n}{2\,d} ~\,\forall~ k\in [d] \right\}\nonumber
\end{align}
Since for every $k\in [d],$ $\cI_k$ is a $\mathsf{Bin}(n, 1/d)$ r.v., then by the multiplicative Chernoff's bound and the union bound, we have
\begin{align}
\pr{}{\left(\cI_1, \ldots, \cI_d\right) \notin \gd}\leq d\, e^{-\frac{n}{8\, d}}\label{bound_on_prob_gd}
\end{align}

Now, conditioned on any fixed realization of a  partition $\left(\cI_1, \ldots, \cI_d\right)\in \gd,$  it is easy to see that for every $k\in [d],$ $\bar{y}_k$ is the average of $\hn_k\geq \frac{n}{2d}$ independent r.v.s, each taking a value in $\{-c_{\eps} \, r, ~ c_{\eps}  \,r \}$  w.p. 1. From this observation and using (\ref{cond-expec-bar_y}), it follows that for every $k\in [d],$ ~$\bar{y}_k-\langle \bq_k, \bp\rangle$ is $\sigma^2$-subGaussian, where $\sigma^2=\frac{4\, d\,c_{\eps}^2\, r^2}{n}$. Hence, by a standard fact concerning the expectation of the maximum of subGaussians (see, e.g., \cite{mit}), we have 
\begin{align}
\ex{}{\max\limits_{k\in [d]}~ \lvert~\bar{y}_k-\langle \bq_k, \bp\rangle~\rvert ~~\big\vert ~\left(\cI_1, \ldots, \cI_d\right)\in\gd}\leq 2\,c_{\eps}\,r\,\sqrt{\frac{2\,d\,\log(2d)}{n}}\label{max_cond_expec}
\end{align}

Putting (\ref{max_cond_expec}) and (\ref{bound_on_prob_gd}) together, and noting that the error is always bounded by $r$, we get 
\begin{align}
\ex{}{\max\limits_{k\in [d]} ~\lvert~\bar{y}_k-\langle \bq_k, \bp\rangle~\rvert }&\leq 2\,c_{\eps}\,r\,\sqrt{\frac{2\,d\,\log(2d)}{n}} + r\,d\, e^{-\frac{n}{8\, d}} \label{max_expec}
\end{align}
By the assumption that $n\geq 8\,d\,\log(n)$, the second term on the right-hand side is bounded by $r\, \frac{d}{n}$, and hence, dominated by the first term. This gives the desired bound on $\err_{\prot_{\adapsamp},~ L_{\infty}}(\cQ_{\infty}(r), d, n)$.

\end{proof}

\fi

%% file: main.bbl
\newcommand{\etalchar}[1]{$^{#1}$}
\begin{thebibliography}{DKM{\etalchar{+}}06}

\bibitem[ASZ18]{acharya2018communication}
Jayadev Acharya, Ziteng Sun, and Huanyu Zhang.
\newblock Communication efficient, sample optimal, linear time locally private
  discrete distribution estimation.
\newblock {\em arXiv preprint arXiv:1802.04705}, 2018.

\bibitem[BNS{\etalchar{+}}16]{BNS+16}
Raef Bassily, Kobbi Nissim, Adam Smith, Thomas Steinke, Uri Stemmer, and
  Jonathan Ullman.
\newblock Algorithmic stability for adaptive data analysis.
\newblock In {\em STOC}, 2016.

\bibitem[BNS18]{BNS18}
Mark Bun, Jelani Nelson, and Uri Stemmer.
\newblock Heavy hitters and the structure of local privacy.
\newblock In {\em Proceedings of the 35th ACM SIGMOD-SIGACT-SIGAI Symposium on
  Principles of Database Systems}, pages 435--447. ACM, 2018.

\bibitem[BNST17]{BNST17}
Raef Bassily, Kobbi Nissim, Uri Stemmer, and Abhradeep Thakurta.
\newblock Practical locally private heavy-hitters.
\newblock {\em NIPS}, 2017.

\bibitem[BS15]{BS15}
Raef Bassily and Adam Smith.
\newblock Local, private, efficient protocols for succinct histograms.
\newblock In {\em Proceedings of the Forty-Seventh Annual ACM on Symposium on
  Theory of Computing (STOC)}, pages 127--135. ACM, 2015.

\bibitem[Bul]{buldygin2000metric}
Valeri{\u\i} Buldygin.
\newblock {\em Metric characterization of random variables and random
  processes}.

\bibitem[DHS15]{NIPS2015_5713}
Ilias Diakonikolas, Moritz Hardt, and Ludwig Schmidt.
\newblock Differentially private learning of structured discrete distributions.
\newblock In C.~Cortes, N.~D. Lawrence, D.~D. Lee, M.~Sugiyama, and R.~Garnett,
  editors, {\em Advances in Neural Information Processing Systems 28}, pages
  2566--2574. Curran Associates, Inc., 2015.

\bibitem[DJW13a]{duchi2013local}
John~C Duchi, Michael~I Jordan, and Martin~J Wainwright.
\newblock Local privacy and statistical minimax rates.
\newblock In {\em Foundations of Computer Science (FOCS), 2013 IEEE 54th Annual
  Symposium on}, pages 429--438. IEEE, 2013.

\bibitem[DJW13b]{DJW13}
John~C Duchi, Michael~I Jordan, and Martin~J Wainwright.
\newblock Local privacy, data processing inequalities, and statistical minimax
  rates.
\newblock {\em arXiv preprint arXiv:1302.3203}, 2013.

\bibitem[DKM{\etalchar{+}}06]{DKMMN06}
Cynthia Dwork, Krishnaram Kenthapadi, Frank McSherry, Ilya Mironov, and Moni
  Naor.
\newblock Our data, ourselves: Privacy via distributed noise generation.
\newblock In {\em EUROCRYPT}, 2006.

\bibitem[DMNS06]{DMNS}
Cynthia Dwork, Frank McSherry, Kobbi Nissim, and Adam Smith.
\newblock Calibrating noise to sensitivity in private data analysis.
\newblock In {\em Theory of Cryptography Conference}, pages 265--284. Springer,
  2006.

\bibitem[EPK14]{erlingsson2014rappor}
{\'U}lfar Erlingsson, Vasyl Pihur, and Aleksandra Korolova.
\newblock Rappor: Randomized aggregatable privacy-preserving ordinal response.
\newblock In {\em CCS}, 2014.

\bibitem[HR10]{hardt2010multiplicative}
Moritz Hardt and Guy~N Rothblum.
\newblock A multiplicative weights mechanism for privacy-preserving data
  analysis.
\newblock In {\em Foundations of Computer Science (FOCS), 2010 51st Annual IEEE
  Symposium on}, pages 61--70. IEEE, 2010.

\bibitem[KBR16]{kairouz2016discrete}
Peter Kairouz, Keith Bonawitz, and Daniel Ramage.
\newblock Discrete distribution estimation under local privacy.
\newblock {\em arXiv preprint arXiv:1602.07387}, 2016.

\bibitem[NTZ13]{NTZ12}
Aleksandar Nikolov, Kunal Talwar, and Li~Zhang.
\newblock The geometry of differential privacy: the sparse and approximate
  cases.
\newblock In {\em Proceedings of the forty-fifth annual ACM symposium on Theory
  of computing}, pages 351--360. ACM, 2013.

\bibitem[Rig15]{mit}
Philippe Rigollet.
\newblock {\em Lecture Notes. 18.S997: High Dimensional Statistics}.
\newblock MIT Courses/Mathematics, 2015.
  ~https://ocw.mit.edu/courses/mathematics/18-s997-high-dimensional-statistics-spring-2015.

\bibitem[TVV{\etalchar{+}}17]{thakurta2017learning}
A.G. Thakurta, A.H. Vyrros, U.S. Vaishampayan, G.~Kapoor, J.~Freudiger, V.R.
  Sridhar, and D.~Davidson.
\newblock Learning new words, 2017.
\newblock US Patent 9,594,741.

\bibitem[War65]{W65}
Stanley~L. Warner.
\newblock Randomized response: A survey technique for eliminating evasive
  answer bias.
\newblock {\em Journal of the American Statistical Association},
  60(309):63--69, 1965.

\bibitem[YB18]{ye2018optimal}
Min Ye and Alexander Barg.
\newblock Optimal schemes for discrete distribution estimation under locally
  differential privacy.
\newblock {\em IEEE Transactions on Information Theory}, 2018.

\end{thebibliography}
